\documentclass{article} 
\usepackage{iclr2027_conference,times}

\usepackage[hypertexnames=false]{hyperref}
\usepackage{url}
\usepackage{amsmath,amssymb,amsthm}
\usepackage{booktabs}
\usepackage{tabularx}
\usepackage{enumitem}      
\usepackage{algorithm}
\usepackage{algpseudocode}
\usepackage[capitalize,noabbrev]{cleveref}
\usepackage{tikz}
\usepackage{placeins}   

\theoremstyle{plain}
\newtheorem{theorem}{Theorem}
\newtheorem{proposition}{Proposition}
\theoremstyle{definition}

\newtheorem{remark}{Remark}

\newcommand{\Esig}{E^{\mathrm{sig}}}
\newcommand{\Euniq}{E^{\mathrm{uniq}}}

\newcommand{\Ex}{\mathbb{E}}
\newcommand{\Prob}{\mathbb{P}}

\title{N$\times$N E-valuation: Testing Scope-Level and Unit-Matched Predictive Associations}

\author{Bin Wang, Yan Zhong, Zhengyu Zhang, Liang Luo, Buyun Zhang \& Ellie Wen \\
Meta \\
\texttt{\{binwang88,yzhong36,zhengyuzhang\}@meta.com} \\
\texttt{\{liangluo,buyunz,ellie.wen\}@meta.com}}

\usepackage[T1]{fontenc}  
\iclrfinalcopy 
\begin{document}
\maketitle
\lhead{} 
\providecommand{\Dinv}{\delta^{\ast}_{\mathrm{inv}}}
\providecommand{\Dvar}{\delta^{\ast}_{\mathrm{var}}}
\providecommand{\ejsig}{e^{\mathrm{sig}}_j}
\providecommand{\eijuniq}{e^{\mathrm{uniq}}_{ij}}
\begin{abstract}
Language models now propose hypotheses faster than anyone can check them, so the step that certifies them must
need nothing from the proposer but the hypothesis. We state three requirements for such a step---per-hypothesis automatable, attributable and
diagnostic---and introduce N$\times$N e-valuation, which meets them. A hypothesis names a scope, the units it is about,
and gives each of them a restriction that scores its outcome. The units of one scope serve as nulls for one another, so
no resampling model is written per hypothesis; the algorithm certifies a stable, scope-level association and reports
whether the scope, the restriction, or which unit receives which restriction carries the prediction. For independently
drawn units, betting e-values bound the chance of a false certificate for each hypothesis in finite samples, under
stated nulls. A betting lower bound on the share of units that carry the association gates every certificate, and an
asymptotically valid equivalence test calls the restrictions shared only when it rejects differences between them above
a declared noise ceiling. We test it on planted data, on randomized-intervention data, and on real purchase data with
$285$ hypotheses that language models wrote without knowing our method or our classes, labeled by human experts.
Meeting all three requirements costs no accuracy: N$\times$N makes
$692$ correct decisions in $693$ with one false certificate, against $659$ with four for a composite of CRT tests that
returns our verdicts but needs a resampling model per hypothesis, and $224$ with $468$ for
held-out validation. Removing any of
our statistics merges a verdict that certifies with one that rejects. The certificates concern predictive association,
not causality.
\end{abstract}

\section{Introduction}\label{sec:intro}
Large language models now propose hypotheses faster than anyone can check them, and many are wrong
\citep{alkan2025hypgen}. An LLM-based system needs a certification step that asks of the proposer only
what it does reliably---stating a hypothesis---and returns feedback, since LLMs rarely correct their
own reasoning unaided \citep{huang2024cannot}. Hence we propose three requirements for an
LLM-compatible certification algorithm---per-hypothesis automatable, attributable and diagnostic---and
N$\times$N e-valuation, an algorithm that meets them.

A certification algorithm is \emph{per-hypothesis automatable} if, once a hypothesis is stated, a machine can carry
out every per-hypothesis step with the guarantee intact. Stating it---code that selects the scope and builds each
unit's restriction---is the proposer's step, which every procedure needs; code that says something other than what was meant
states a different hypothesis, and that one is still tested validly. A conditional randomization test (CRT) is not
per-hypothesis automatable, in general: it also needs a resampling model written for each hypothesis
\citep{candes2018}, which an LLM can write but cannot
ensure is right, and a wrong model silently voids the guarantee. An algorithm is \emph{attributable} if what it
certifies is a stable, scope-level hypothesis association: across the scope, the hypothesis, not
prevalence, carries the prediction---an association, not causality. Held-out validation is not: a held-out
gain can be a restated base rate \citep{dacrema2019progress} or popularity bias
\citep{abdollahpouri2019popularity}. An algorithm is \emph{diagnostic} if its output helps the proposer revise a hypothesis or keep the
part that works, for instance by saying which part of it---the scope, the restriction or their
pairing to units---carries the association. A CRT and a permutation test are not: both return only
certify or reject. None of these meets all three.

N$\times$N e-valuation is proposed to meet all three. It is per-hypothesis automatable because the units
under a hypothesis serve as nulls for one another, replacing the one step of a CRT that is not: the
construction of a null for each hypothesis. It is attributable as a CRT is: it certifies only when altering a
component of the hypothesis---its scope, or which unit receives which restriction---changes the
prediction, and what it certifies is a claim about the whole scope (\cref{sec:method}). It is
diagnostic because its verdict is richer than certify or reject: its verdicts distinguish the failures
most common in the field---prevalence, a scope-level effect read as individual \citep{robinson1950} and a
lift carried by a subgroup \citep{simpson1951,jadidinejad2021simpson}---and each suggests how to revise the
hypothesis (\cref{tab:verdictuse} in \cref{app:verdicts}).

Our contributions, for datasets of independently drawn units, are two: the requirements and
N$\times$N e-valuation, to our knowledge the first algorithm meeting all three. No prior algorithm targets them, so our
reference procedures (hereafter \emph{procedures}) are two. A composite we build by routing three CRT tests and a
finite-sample invariance test through our verdicts is diagnostic but not per-hypothesis automatable, since it needs a
resampling model per hypothesis; held-out validation is automatable but neither attributable nor diagnostic. We ask
whether needing no per-hypothesis resampling model costs correct scope-level decisions, and it costs none. On planted and
randomized-intervention data, and on real purchase data with $285$ hypotheses that language models wrote knowing
neither our method nor our classes, labeled by human experts, N$\times$N makes $692$ correct decisions in $693$ with
one false certificate, against $659$ with four for the composite, which declines $17$ of the $21$ shared associations
on the real data; held-out validation performs far worse than both, with $224$ correct decisions and $468$ false
certificates.

\section{Related work}\label{sec:related}
\textbf{Conformal e-values.} The e-value is the safe-testing construction
\citep{grunwald2024safe,ramdas2023}, and scoring an observation against an exchangeable pool is its
conformal instance \citep{vovk2021,vovk2015cross,bates2023outliers}. Cross-conformal prediction and
the jackknife+ \citep{vovk2015cross,barber2021jackknife} reuse a pool to calibrate a predictor, not
to test a hypothesis; conformal test martingales \citep{vovk2021online} test whether a \emph{stream}
stays exchangeable, not a fixed pairing on a set in hand. $\Esig$'s per-unit terms are that standard construction and
$\Euniq$'s the group-invariance e-value for a permutation null \citep{koning2023,perezortiz2022};
both are combined by betting \citep{shafer2021testing,waudbysmith2024estimating}. A mean merges
e-values under arbitrary dependence and a product under independence \citep{vovk2021}.

\textbf{Machine proposers and their verifiers.} Discovery benchmarks supply labelled instances to
\emph{measure} verification against, not a way to perform it on a new proposal
\citep{majumder2024discoverybench}. Self-verification is cheapest and inherits the proposer's
failure mode: models do not reliably correct their own reasoning without an external signal
\citep{huang2024cannot,kamoi2024when}. Trained verifiers \citep{lightman2023verify} supply one, but
for \emph{reasoning steps} rather than a claim about a population. Both remain unbounded model
judgements, so we compare against tests run on the same data.

\textbf{Randomisation tests.} The permutation test \citep{fisher1935,pitman1937} draws its
reference distribution from exchangeability rather than from a model. The model-X CRT
\citep{candes2018} resamples the covariate from
$P(X_j\mid X_{-j})$; the conditional permutation test \citep{berrett2020cpt} relaxes exactness of
that law, and sequential model-X \citep{shaer2023modelx} recasts the family as betting. Some
restriction on the covariate law is unavoidable---without one, no test of conditional independence
has nontrivial power against all alternatives \citep{shah2020hardness}---and model-X supplies it by
assumption. All three need that conditional structure per hypothesis; the pairing null needs none,
and no such test can express a unit-invariant hypothesis at any budget (\cref{app:crtlimits}).

\textbf{Matrix nulls.} \citet{mantel1967} permutes the pairing between two distance matrices;
\citet{xu2024learning} build a $K\times K$ matrix from a learned critic and test its diagonal mean
against column permutations. Both share the \emph{mutual-null} construction, samples serving as nulls
for each other. We arrived at it independently, for hypotheses of any form rather than one fixed form, and $\Euniq$ asks their question---whether the diagonal stands out---pair by pair. They use the off-diagonal only for their own purposes; we read it with
three statistics: $\Euniq$ against swapped pairings, the invariance call on the variance structure,
the certified share on the per-unit split. \cref{tab:necessity} derives which verdicts merge without
each; a permutation test on the same matrix, lacking all three, rejects most \emph{Shared} samples.
To our knowledge that combination is unoccupied.

\textbf{Other work.} Significant pattern mining asks the same question with multiplicity
control: Tarone's testability bound \citep{tarone1990modified} excludes patterns that cannot reach
significance, LAMP \citep{terada2013lamp} makes that a corrected threshold, and
Westfall--Young light \citep{llinares2015westfall} reaches FWER control by permutation at scale.
All of them are coupled to the enumeration.

\section{Method: The \texorpdfstring{$N\times N$}{NxN} E-Value Matrix Certifier}\label{sec:method}
A hypothesis is a pair $H=(\mathrm{in\_scope},\mathrm{apply})$ of functions of the input:
$\mathrm{in\_scope}(x)$ decides membership, $\mathrm{apply}$ returns a \emph{restriction} $r$ scoring
an outcome as $r(y)\ge0$. Only $H$ is written for a hypothesis; the input, the outcome and how a restriction scores
it are fixed once per dataset. \cref{tab:objects} (\cref{app:notation}) fixes the notation and \cref{alg:certify} the
steps. Every quantity here can be statistical---restrictions may be sampled, scores estimated,
outcomes noisy---so the thresholds are tolerance dials the user sets (\cref{app:tolerances}).
\begin{algorithm}[t]
\caption{$N\times N$ E-valuation: certify hypothesis $H$ at level $\alpha$}
\label{alg:certify}
\begin{algorithmic}[1]
\Require $H=(\mathrm{in\_scope},\mathrm{apply})$, data $\{(x_i,y_i)\}_{i\in T}$, level $\alpha$, thresholds $\tau_{\mathrm{cov}},\gamma$, noise ceiling $k$
\State $U \gets \{\, i\in T : \mathrm{in\_scope}(x_i)=1 \,\}$ \Comment{in-scope subset $U\subseteq T$}
\State \textbf{if} $|U|>\gamma|T|$ \textbf{return} \emph{base rate} \Comment{a scope over $\gamma$ of $T$ is the population}
\State $V \gets |V|$ units drawn uniformly from $T\setminus U$ \Comment{out-of-scope reference, $|V|\ge|U|$}
\State split $V$ into disjoint blocks $B_i$, $i\in U$, of $\ell=\lfloor|V|/|U|\rfloor$ units \Comment{no outcome read}
\For{$i \in U$}
    \State $r_i \gets \mathrm{apply}(x_i)$ \Comment{restriction; normally the expensive call, once per unit ($O(|U|)$)}
\EndFor
\For{$i \in U,\ j \in U$}
    \State $M[i][j] \gets r_i(y_j)$ \Comment{cheap score; diagonal $M[i][i]=r_i(y_i)$ ($O(|U|^2)$)}
\EndFor
\State $\Esig \gets \mathrm{bet}\big(\{1+r_i(y_i)-\bar r_i\}_{i\in U}\big)$ \Comment{significance: own outcome vs.\ its block's mean $\bar r_i$}
\State pair $U$'s units at random; $\Euniq \gets \mathrm{bet}\big(\{\eijuniq\}\big)$ \Comment{uniqueness: own pairing vs.\ swapped}
\State $\widehat D,\ \widehat{\mathrm{se}} \gets \widehat S_{\mathrm r}-(k\bar M)^2$ off the diagonal, and its jackknife error \Comment{\cref{sec:invariance}; $O(|U|^2)$}
\State $\mathrm{shared}\gets[\,\widehat D+z_{1-\alpha}\widehat{\mathrm{se}}<0\,]$ \Comment{the invariance call (\cref{eq:call})}
\State $s_j \gets \operatorname{sign}\big(r_j(y_j)-r_j(y_{v_j})\big)$, distinct random $v_j\in V$; $\hat c\gets$ certified share \Comment{\cref{eq:cover}}
\State \Return $\mathrm{verdict}(\Esig,\Euniq,\mathrm{shared},\hat c)$ per \cref{sec:decision} \Comment{certify iff \emph{Matched} or \emph{Shared}}
\end{algorithmic}
\end{algorithm}

\subsection{Concepts}\label{sec:statistics}
Two nulls replace the per-hypothesis resampling model, and neither is written for a hypothesis. The significance null
is no lift over the units outside the scope; it compares mean fit and nothing else, so a gain from any
source---the restriction, its pairing or the scope itself---counts against it. The uniqueness null is an exchangeable
pairing within the scope. What the method asks of the data is stated once per dataset, not per hypothesis:
independently drawn units (\cref{app:theory}).
\paragraph{Significance null.} Let the \emph{reference} $V$ be units drawn uniformly from
outside the scope, $T\setminus U$, without reading an outcome---at least $|U|$ of them---and give each $j\in U$ its
own block $B_j\subset V$ of $\ell=\lfloor|V|/|U|\rfloor$ of them, the blocks disjoint. $H_0^{\mathrm{sig}}$ (no lift):
a unit's restriction fits its own outcome no better, in expectation, than an outcome from outside the scope,
$\Ex[r_j(y_j)]\le\Ex[r_j(y_v)]$ for $j$ drawn from the scope and $v$ from outside it.

\paragraph{Uniqueness null.} $H_0^{\mathrm{uniq}}$: given the restrictions $R_U$ and outcomes $Y_U$ of the scope, the pairing of restrictions to outcomes is exchangeable.

\paragraph{Significance e-value.}
For unit $j$, with $r_j=\mathrm{apply}(x_j)$ and scores in $[0,1]$ (recall's range; other scores are divided by
their bound), a per-unit term scores its own outcome against its block's mean:
\begin{equation}\label{eq:esig}
\ejsig \;=\; 1+r_j(y_j)-\bar r_j, \qquad
\bar r_j \;=\; \frac{1}{\ell}\sum_{k\in B_j} r_j(y_k),
\end{equation}
so $0\le\ejsig\le2$ and, under $H_0^{\mathrm{sig}}$, $\Ex[\ejsig]\le1$. The term is a difference, not the
ratio a lift suggests: a ratio's random denominator can push its mean above $1$ under no lift
(\cref{app:theory,app:nullcheck}).

\paragraph{Uniqueness --- the $N\times N$ matrix.}\label{sec:mech}
Over the claimed population (size $|U|$), compute each restriction \emph{once}, then form
\begin{equation}\label{eq:matrixdef}
M[i][j] \;=\; r_i(y_j).
\end{equation}
Pair the units of $U$ at random, without reading an outcome. For a pair $(i,j)$ the \emph{conformal uniqueness
term}---the group-invariance e-value for swapping the pair's outcomes \citep{koning2023,perezortiz2022}---weighs the
pair's own pairing against the swapped one:
\begin{equation}\label{eq:emech}
\eijuniq \;=\; \frac{2\,(M[i][i]+M[j][j])}{M[i][i]+M[j][j]+M[i][j]+M[j][i]}\;\in[0,2].
\end{equation}
A pair whose four cells are all zero gives $0/0$, set to the null value $1$.

\paragraph{Invariance.} A hypothesis is \emph{invariant} over its scope when its units share one
restriction; every column of $M$ then applies that restriction to one outcome and is constant up to
the restriction's own variation across units. The user declares a ceiling on that variation, $k$
times the mean off-diagonal score $\bar M$: restrictions whose scores on a common outcome spread by less
are one restriction with noise. The invariance call tests the mean variance of $M$'s columns against the
squared ceiling (\cref{sec:invariance}): \emph{shared} when it is shown to lie below, \emph{variant} when
above, and \emph{undecided} otherwise.

\paragraph{Coverage.} Each unit $j$ is paired with a comparator $v_j$, a distinct member of $V$ drawn at random, and
scores $s_j=+1$ when its own outcome fits its restriction better than the comparator's outcome does, $-1$ when worse
and $0$ on a tie. A unit has a gain when it is likelier to win than to lose, so the expected mean of the $s_j$ is at most
the share of units with a gain, and the \emph{certified share} $\hat c$ bounds that share from below: it is the largest
$\tau$ at which
\begin{equation}\label{eq:cover}
\frac{1}{|\Lambda|}\sum_{\lambda\in\Lambda}\ \prod_{j\in U}\Big(1+\lambda\,\frac{s_j-\tau}{1+\tau}\Big)\;\ge\;\frac1\alpha,
\qquad s_j=\operatorname{sign}\big(r_j(y_j)-r_j(y_{v_j})\big).
\end{equation}
Coverage is shown when $\hat c\ge\tau_{\mathrm{cov}}$, the share of units the user requires to carry the gain.

\paragraph{Reading the four.} \cref{fig:matrix} shows $M$ under the four regimes. Both e-values ask whether a restriction
scores its own unit's outcome above other units' outcomes, and differ only in the pool those units are drawn from. $\Esig$ draws from
outside the scope, among units the hypothesis does not claim, and accumulates evidence only when the restrictions fit
in-scope outcomes better on average than out-of-scope ones, whatever makes the two populations differ: a certificate
is a scope-level association, not evidence that the hypothesis names its cause. The scope is no part of its own
baseline, so its width does not cap $\Esig$ (\cref{app:gamma}).
$\Euniq$ draws from the scope $U$, randomizing the \emph{pairing} rather than the covariate, so
unlike the model-X CRT \citep{candes2018} it needs no model of $P(X_j\mid X_{-j})$.
The invariance call asks whether one restriction explains every unit. And the certified share reads the diagonal
per unit rather than in aggregate: unless it shows that more than a fraction $\tau_{\mathrm{cov}}$ of units carry the
gain, the certificate is withheld and the case named \emph{Sparse}.

\begin{figure}[t]
\centering
\begin{tikzpicture}[scale=0.86]
\node[anchor=south,font=\scriptsize\bfseries] at (1.68,3.78) {Matched};
\node[anchor=south,font=\scriptsize] at (1.68,3.41) {bright diagonal};
\fill[black!80] (0.00,3.12) rectangle ++(0.23,0.23);
\fill[black!0] (0.24,3.12) rectangle ++(0.23,0.23);
\fill[black!0] (0.48,3.12) rectangle ++(0.23,0.23);
\fill[black!0] (0.72,3.12) rectangle ++(0.23,0.23);
\fill[black!0] (0.96,3.12) rectangle ++(0.23,0.23);
\fill[black!0] (1.20,3.12) rectangle ++(0.23,0.23);
\fill[black!0] (1.44,3.12) rectangle ++(0.23,0.23);
\fill[black!0] (1.68,3.12) rectangle ++(0.23,0.23);
\fill[black!0] (1.92,3.12) rectangle ++(0.23,0.23);
\fill[black!0] (2.16,3.12) rectangle ++(0.23,0.23);
\fill[black!0] (2.40,3.12) rectangle ++(0.23,0.23);
\fill[black!0] (2.64,3.12) rectangle ++(0.23,0.23);
\fill[black!0] (2.88,3.12) rectangle ++(0.23,0.23);
\fill[black!0] (3.12,3.12) rectangle ++(0.23,0.23);
\fill[black!0] (0.00,2.88) rectangle ++(0.23,0.23);
\fill[black!60] (0.24,2.88) rectangle ++(0.23,0.23);
\fill[black!0] (0.48,2.88) rectangle ++(0.23,0.23);
\fill[black!0] (0.72,2.88) rectangle ++(0.23,0.23);
\fill[black!0] (0.96,2.88) rectangle ++(0.23,0.23);
\fill[black!0] (1.20,2.88) rectangle ++(0.23,0.23);
\fill[black!0] (1.44,2.88) rectangle ++(0.23,0.23);
\fill[black!0] (1.68,2.88) rectangle ++(0.23,0.23);
\fill[black!0] (1.92,2.88) rectangle ++(0.23,0.23);
\fill[black!0] (2.16,2.88) rectangle ++(0.23,0.23);
\fill[black!0] (2.40,2.88) rectangle ++(0.23,0.23);
\fill[black!0] (2.64,2.88) rectangle ++(0.23,0.23);
\fill[black!0] (2.88,2.88) rectangle ++(0.23,0.23);
\fill[black!0] (3.12,2.88) rectangle ++(0.23,0.23);
\fill[black!0] (0.00,2.64) rectangle ++(0.23,0.23);
\fill[black!0] (0.24,2.64) rectangle ++(0.23,0.23);
\fill[black!60] (0.48,2.64) rectangle ++(0.23,0.23);
\fill[black!0] (0.72,2.64) rectangle ++(0.23,0.23);
\fill[black!0] (0.96,2.64) rectangle ++(0.23,0.23);
\fill[black!0] (1.20,2.64) rectangle ++(0.23,0.23);
\fill[black!0] (1.44,2.64) rectangle ++(0.23,0.23);
\fill[black!0] (1.68,2.64) rectangle ++(0.23,0.23);
\fill[black!20] (1.92,2.64) rectangle ++(0.23,0.23);
\fill[black!0] (2.16,2.64) rectangle ++(0.23,0.23);
\fill[black!0] (2.40,2.64) rectangle ++(0.23,0.23);
\fill[black!0] (2.64,2.64) rectangle ++(0.23,0.23);
\fill[black!0] (2.88,2.64) rectangle ++(0.23,0.23);
\fill[black!20] (3.12,2.64) rectangle ++(0.23,0.23);
\fill[black!0] (0.00,2.40) rectangle ++(0.23,0.23);
\fill[black!0] (0.24,2.40) rectangle ++(0.23,0.23);
\fill[black!0] (0.48,2.40) rectangle ++(0.23,0.23);
\fill[black!60] (0.72,2.40) rectangle ++(0.23,0.23);
\fill[black!0] (0.96,2.40) rectangle ++(0.23,0.23);
\fill[black!0] (1.20,2.40) rectangle ++(0.23,0.23);
\fill[black!0] (1.44,2.40) rectangle ++(0.23,0.23);
\fill[black!0] (1.68,2.40) rectangle ++(0.23,0.23);
\fill[black!0] (1.92,2.40) rectangle ++(0.23,0.23);
\fill[black!0] (2.16,2.40) rectangle ++(0.23,0.23);
\fill[black!0] (2.40,2.40) rectangle ++(0.23,0.23);
\fill[black!0] (2.64,2.40) rectangle ++(0.23,0.23);
\fill[black!0] (2.88,2.40) rectangle ++(0.23,0.23);
\fill[black!0] (3.12,2.40) rectangle ++(0.23,0.23);
\fill[black!0] (0.00,2.16) rectangle ++(0.23,0.23);
\fill[black!0] (0.24,2.16) rectangle ++(0.23,0.23);
\fill[black!0] (0.48,2.16) rectangle ++(0.23,0.23);
\fill[black!0] (0.72,2.16) rectangle ++(0.23,0.23);
\fill[black!60] (0.96,2.16) rectangle ++(0.23,0.23);
\fill[black!0] (1.20,2.16) rectangle ++(0.23,0.23);
\fill[black!0] (1.44,2.16) rectangle ++(0.23,0.23);
\fill[black!0] (1.68,2.16) rectangle ++(0.23,0.23);
\fill[black!0] (1.92,2.16) rectangle ++(0.23,0.23);
\fill[black!0] (2.16,2.16) rectangle ++(0.23,0.23);
\fill[black!0] (2.40,2.16) rectangle ++(0.23,0.23);
\fill[black!0] (2.64,2.16) rectangle ++(0.23,0.23);
\fill[black!0] (2.88,2.16) rectangle ++(0.23,0.23);
\fill[black!0] (3.12,2.16) rectangle ++(0.23,0.23);
\fill[black!0] (0.00,1.92) rectangle ++(0.23,0.23);
\fill[black!0] (0.24,1.92) rectangle ++(0.23,0.23);
\fill[black!0] (0.48,1.92) rectangle ++(0.23,0.23);
\fill[black!0] (0.72,1.92) rectangle ++(0.23,0.23);
\fill[black!0] (0.96,1.92) rectangle ++(0.23,0.23);
\fill[black!60] (1.20,1.92) rectangle ++(0.23,0.23);
\fill[black!0] (1.44,1.92) rectangle ++(0.23,0.23);
\fill[black!0] (1.68,1.92) rectangle ++(0.23,0.23);
\fill[black!0] (1.92,1.92) rectangle ++(0.23,0.23);
\fill[black!0] (2.16,1.92) rectangle ++(0.23,0.23);
\fill[black!0] (2.40,1.92) rectangle ++(0.23,0.23);
\fill[black!0] (2.64,1.92) rectangle ++(0.23,0.23);
\fill[black!0] (2.88,1.92) rectangle ++(0.23,0.23);
\fill[black!0] (3.12,1.92) rectangle ++(0.23,0.23);
\fill[black!0] (0.00,1.68) rectangle ++(0.23,0.23);
\fill[black!0] (0.24,1.68) rectangle ++(0.23,0.23);
\fill[black!0] (0.48,1.68) rectangle ++(0.23,0.23);
\fill[black!0] (0.72,1.68) rectangle ++(0.23,0.23);
\fill[black!0] (0.96,1.68) rectangle ++(0.23,0.23);
\fill[black!0] (1.20,1.68) rectangle ++(0.23,0.23);
\fill[black!80] (1.44,1.68) rectangle ++(0.23,0.23);
\fill[black!0] (1.68,1.68) rectangle ++(0.23,0.23);
\fill[black!0] (1.92,1.68) rectangle ++(0.23,0.23);
\fill[black!0] (2.16,1.68) rectangle ++(0.23,0.23);
\fill[black!0] (2.40,1.68) rectangle ++(0.23,0.23);
\fill[black!0] (2.64,1.68) rectangle ++(0.23,0.23);
\fill[black!0] (2.88,1.68) rectangle ++(0.23,0.23);
\fill[black!0] (3.12,1.68) rectangle ++(0.23,0.23);
\fill[black!0] (0.00,1.44) rectangle ++(0.23,0.23);
\fill[black!0] (0.24,1.44) rectangle ++(0.23,0.23);
\fill[black!0] (0.48,1.44) rectangle ++(0.23,0.23);
\fill[black!0] (0.72,1.44) rectangle ++(0.23,0.23);
\fill[black!0] (0.96,1.44) rectangle ++(0.23,0.23);
\fill[black!0] (1.20,1.44) rectangle ++(0.23,0.23);
\fill[black!0] (1.44,1.44) rectangle ++(0.23,0.23);
\fill[black!60] (1.68,1.44) rectangle ++(0.23,0.23);
\fill[black!0] (1.92,1.44) rectangle ++(0.23,0.23);
\fill[black!0] (2.16,1.44) rectangle ++(0.23,0.23);
\fill[black!0] (2.40,1.44) rectangle ++(0.23,0.23);
\fill[black!0] (2.64,1.44) rectangle ++(0.23,0.23);
\fill[black!0] (2.88,1.44) rectangle ++(0.23,0.23);
\fill[black!0] (3.12,1.44) rectangle ++(0.23,0.23);
\fill[black!0] (0.00,1.20) rectangle ++(0.23,0.23);
\fill[black!0] (0.24,1.20) rectangle ++(0.23,0.23);
\fill[black!20] (0.48,1.20) rectangle ++(0.23,0.23);
\fill[black!0] (0.72,1.20) rectangle ++(0.23,0.23);
\fill[black!0] (0.96,1.20) rectangle ++(0.23,0.23);
\fill[black!0] (1.20,1.20) rectangle ++(0.23,0.23);
\fill[black!0] (1.44,1.20) rectangle ++(0.23,0.23);
\fill[black!0] (1.68,1.20) rectangle ++(0.23,0.23);
\fill[black!80] (1.92,1.20) rectangle ++(0.23,0.23);
\fill[black!0] (2.16,1.20) rectangle ++(0.23,0.23);
\fill[black!0] (2.40,1.20) rectangle ++(0.23,0.23);
\fill[black!0] (2.64,1.20) rectangle ++(0.23,0.23);
\fill[black!0] (2.88,1.20) rectangle ++(0.23,0.23);
\fill[black!0] (3.12,1.20) rectangle ++(0.23,0.23);
\fill[black!0] (0.00,0.96) rectangle ++(0.23,0.23);
\fill[black!0] (0.24,0.96) rectangle ++(0.23,0.23);
\fill[black!0] (0.48,0.96) rectangle ++(0.23,0.23);
\fill[black!0] (0.72,0.96) rectangle ++(0.23,0.23);
\fill[black!0] (0.96,0.96) rectangle ++(0.23,0.23);
\fill[black!0] (1.20,0.96) rectangle ++(0.23,0.23);
\fill[black!0] (1.44,0.96) rectangle ++(0.23,0.23);
\fill[black!0] (1.68,0.96) rectangle ++(0.23,0.23);
\fill[black!0] (1.92,0.96) rectangle ++(0.23,0.23);
\fill[black!0] (2.16,0.96) rectangle ++(0.23,0.23);
\fill[black!0] (2.40,0.96) rectangle ++(0.23,0.23);
\fill[black!0] (2.64,0.96) rectangle ++(0.23,0.23);
\fill[black!0] (2.88,0.96) rectangle ++(0.23,0.23);
\fill[black!0] (3.12,0.96) rectangle ++(0.23,0.23);
\fill[black!0] (0.00,0.72) rectangle ++(0.23,0.23);
\fill[black!0] (0.24,0.72) rectangle ++(0.23,0.23);
\fill[black!0] (0.48,0.72) rectangle ++(0.23,0.23);
\fill[black!0] (0.72,0.72) rectangle ++(0.23,0.23);
\fill[black!0] (0.96,0.72) rectangle ++(0.23,0.23);
\fill[black!0] (1.20,0.72) rectangle ++(0.23,0.23);
\fill[black!0] (1.44,0.72) rectangle ++(0.23,0.23);
\fill[black!0] (1.68,0.72) rectangle ++(0.23,0.23);
\fill[black!0] (1.92,0.72) rectangle ++(0.23,0.23);
\fill[black!0] (2.16,0.72) rectangle ++(0.23,0.23);
\fill[black!80] (2.40,0.72) rectangle ++(0.23,0.23);
\fill[black!0] (2.64,0.72) rectangle ++(0.23,0.23);
\fill[black!0] (2.88,0.72) rectangle ++(0.23,0.23);
\fill[black!0] (3.12,0.72) rectangle ++(0.23,0.23);
\fill[black!0] (0.00,0.48) rectangle ++(0.23,0.23);
\fill[black!0] (0.24,0.48) rectangle ++(0.23,0.23);
\fill[black!0] (0.48,0.48) rectangle ++(0.23,0.23);
\fill[black!0] (0.72,0.48) rectangle ++(0.23,0.23);
\fill[black!0] (0.96,0.48) rectangle ++(0.23,0.23);
\fill[black!0] (1.20,0.48) rectangle ++(0.23,0.23);
\fill[black!0] (1.44,0.48) rectangle ++(0.23,0.23);
\fill[black!0] (1.68,0.48) rectangle ++(0.23,0.23);
\fill[black!0] (1.92,0.48) rectangle ++(0.23,0.23);
\fill[black!0] (2.16,0.48) rectangle ++(0.23,0.23);
\fill[black!20] (2.40,0.48) rectangle ++(0.23,0.23);
\fill[black!60] (2.64,0.48) rectangle ++(0.23,0.23);
\fill[black!0] (2.88,0.48) rectangle ++(0.23,0.23);
\fill[black!0] (3.12,0.48) rectangle ++(0.23,0.23);
\fill[black!0] (0.00,0.24) rectangle ++(0.23,0.23);
\fill[black!0] (0.24,0.24) rectangle ++(0.23,0.23);
\fill[black!0] (0.48,0.24) rectangle ++(0.23,0.23);
\fill[black!0] (0.72,0.24) rectangle ++(0.23,0.23);
\fill[black!0] (0.96,0.24) rectangle ++(0.23,0.23);
\fill[black!0] (1.20,0.24) rectangle ++(0.23,0.23);
\fill[black!0] (1.44,0.24) rectangle ++(0.23,0.23);
\fill[black!0] (1.68,0.24) rectangle ++(0.23,0.23);
\fill[black!0] (1.92,0.24) rectangle ++(0.23,0.23);
\fill[black!0] (2.16,0.24) rectangle ++(0.23,0.23);
\fill[black!0] (2.40,0.24) rectangle ++(0.23,0.23);
\fill[black!0] (2.64,0.24) rectangle ++(0.23,0.23);
\fill[black!60] (2.88,0.24) rectangle ++(0.23,0.23);
\fill[black!0] (3.12,0.24) rectangle ++(0.23,0.23);
\fill[black!0] (0.00,0.00) rectangle ++(0.23,0.23);
\fill[black!0] (0.24,0.00) rectangle ++(0.23,0.23);
\fill[black!0] (0.48,0.00) rectangle ++(0.23,0.23);
\fill[black!0] (0.72,0.00) rectangle ++(0.23,0.23);
\fill[black!0] (0.96,0.00) rectangle ++(0.23,0.23);
\fill[black!0] (1.20,0.00) rectangle ++(0.23,0.23);
\fill[black!0] (1.44,0.00) rectangle ++(0.23,0.23);
\fill[black!0] (1.68,0.00) rectangle ++(0.23,0.23);
\fill[black!0] (1.92,0.00) rectangle ++(0.23,0.23);
\fill[black!0] (2.16,0.00) rectangle ++(0.23,0.23);
\fill[black!0] (2.40,0.00) rectangle ++(0.23,0.23);
\fill[black!0] (2.64,0.00) rectangle ++(0.23,0.23);
\fill[black!0] (2.88,0.00) rectangle ++(0.23,0.23);
\fill[black!60] (3.12,0.00) rectangle ++(0.23,0.23);
\draw[gray!60] (-0.01,-0.01) rectangle (3.36,3.36);
\node[anchor=south,font=\scriptsize\bfseries] at (5.23,3.78) {Shared};
\node[anchor=south,font=\scriptsize] at (5.23,3.41) {identical rows};
\fill[black!64] (3.55,3.12) rectangle ++(0.23,0.23);
\fill[black!56] (3.79,3.12) rectangle ++(0.23,0.23);
\fill[black!64] (4.03,3.12) rectangle ++(0.23,0.23);
\fill[black!0] (4.27,3.12) rectangle ++(0.23,0.23);
\fill[black!45] (4.51,3.12) rectangle ++(0.23,0.23);
\fill[black!61] (4.75,3.12) rectangle ++(0.23,0.23);
\fill[black!45] (4.99,3.12) rectangle ++(0.23,0.23);
\fill[black!0] (5.23,3.12) rectangle ++(0.23,0.23);
\fill[black!64] (5.47,3.12) rectangle ++(0.23,0.23);
\fill[black!69] (5.71,3.12) rectangle ++(0.23,0.23);
\fill[black!40] (5.95,3.12) rectangle ++(0.23,0.23);
\fill[black!56] (6.19,3.12) rectangle ++(0.23,0.23);
\fill[black!72] (6.43,3.12) rectangle ++(0.23,0.23);
\fill[black!80] (6.67,3.12) rectangle ++(0.23,0.23);
\fill[black!64] (3.55,2.88) rectangle ++(0.23,0.23);
\fill[black!64] (3.79,2.88) rectangle ++(0.23,0.23);
\fill[black!72] (4.03,2.88) rectangle ++(0.23,0.23);
\fill[black!0] (4.27,2.88) rectangle ++(0.23,0.23);
\fill[black!43] (4.51,2.88) rectangle ++(0.23,0.23);
\fill[black!69] (4.75,2.88) rectangle ++(0.23,0.23);
\fill[black!45] (4.99,2.88) rectangle ++(0.23,0.23);
\fill[black!5] (5.23,2.88) rectangle ++(0.23,0.23);
\fill[black!69] (5.47,2.88) rectangle ++(0.23,0.23);
\fill[black!64] (5.71,2.88) rectangle ++(0.23,0.23);
\fill[black!43] (5.95,2.88) rectangle ++(0.23,0.23);
\fill[black!48] (6.19,2.88) rectangle ++(0.23,0.23);
\fill[black!56] (6.43,2.88) rectangle ++(0.23,0.23);
\fill[black!80] (6.67,2.88) rectangle ++(0.23,0.23);
\fill[black!67] (3.55,2.64) rectangle ++(0.23,0.23);
\fill[black!64] (3.79,2.64) rectangle ++(0.23,0.23);
\fill[black!59] (4.03,2.64) rectangle ++(0.23,0.23);
\fill[black!0] (4.27,2.64) rectangle ++(0.23,0.23);
\fill[black!56] (4.51,2.64) rectangle ++(0.23,0.23);
\fill[black!67] (4.75,2.64) rectangle ++(0.23,0.23);
\fill[black!48] (4.99,2.64) rectangle ++(0.23,0.23);
\fill[black!0] (5.23,2.64) rectangle ++(0.23,0.23);
\fill[black!59] (5.47,2.64) rectangle ++(0.23,0.23);
\fill[black!56] (5.71,2.64) rectangle ++(0.23,0.23);
\fill[black!51] (5.95,2.64) rectangle ++(0.23,0.23);
\fill[black!43] (6.19,2.64) rectangle ++(0.23,0.23);
\fill[black!69] (6.43,2.64) rectangle ++(0.23,0.23);
\fill[black!83] (6.67,2.64) rectangle ++(0.23,0.23);
\fill[black!56] (3.55,2.40) rectangle ++(0.23,0.23);
\fill[black!59] (3.79,2.40) rectangle ++(0.23,0.23);
\fill[black!69] (4.03,2.40) rectangle ++(0.23,0.23);
\fill[black!0] (4.27,2.40) rectangle ++(0.23,0.23);
\fill[black!48] (4.51,2.40) rectangle ++(0.23,0.23);
\fill[black!64] (4.75,2.40) rectangle ++(0.23,0.23);
\fill[black!48] (4.99,2.40) rectangle ++(0.23,0.23);
\fill[black!0] (5.23,2.40) rectangle ++(0.23,0.23);
\fill[black!64] (5.47,2.40) rectangle ++(0.23,0.23);
\fill[black!64] (5.71,2.40) rectangle ++(0.23,0.23);
\fill[black!48] (5.95,2.40) rectangle ++(0.23,0.23);
\fill[black!48] (6.19,2.40) rectangle ++(0.23,0.23);
\fill[black!56] (6.43,2.40) rectangle ++(0.23,0.23);
\fill[black!80] (6.67,2.40) rectangle ++(0.23,0.23);
\fill[black!67] (3.55,2.16) rectangle ++(0.23,0.23);
\fill[black!69] (3.79,2.16) rectangle ++(0.23,0.23);
\fill[black!72] (4.03,2.16) rectangle ++(0.23,0.23);
\fill[black!0] (4.27,2.16) rectangle ++(0.23,0.23);
\fill[black!48] (4.51,2.16) rectangle ++(0.23,0.23);
\fill[black!72] (4.75,2.16) rectangle ++(0.23,0.23);
\fill[black!48] (4.99,2.16) rectangle ++(0.23,0.23);
\fill[black!8] (5.23,2.16) rectangle ++(0.23,0.23);
\fill[black!64] (5.47,2.16) rectangle ++(0.23,0.23);
\fill[black!69] (5.71,2.16) rectangle ++(0.23,0.23);
\fill[black!43] (5.95,2.16) rectangle ++(0.23,0.23);
\fill[black!48] (6.19,2.16) rectangle ++(0.23,0.23);
\fill[black!72] (6.43,2.16) rectangle ++(0.23,0.23);
\fill[black!80] (6.67,2.16) rectangle ++(0.23,0.23);
\fill[black!67] (3.55,1.92) rectangle ++(0.23,0.23);
\fill[black!69] (3.79,1.92) rectangle ++(0.23,0.23);
\fill[black!72] (4.03,1.92) rectangle ++(0.23,0.23);
\fill[black!0] (4.27,1.92) rectangle ++(0.23,0.23);
\fill[black!56] (4.51,1.92) rectangle ++(0.23,0.23);
\fill[black!64] (4.75,1.92) rectangle ++(0.23,0.23);
\fill[black!56] (4.99,1.92) rectangle ++(0.23,0.23);
\fill[black!0] (5.23,1.92) rectangle ++(0.23,0.23);
\fill[black!64] (5.47,1.92) rectangle ++(0.23,0.23);
\fill[black!64] (5.71,1.92) rectangle ++(0.23,0.23);
\fill[black!48] (5.95,1.92) rectangle ++(0.23,0.23);
\fill[black!48] (6.19,1.92) rectangle ++(0.23,0.23);
\fill[black!61] (6.43,1.92) rectangle ++(0.23,0.23);
\fill[black!72] (6.67,1.92) rectangle ++(0.23,0.23);
\fill[black!61] (3.55,1.68) rectangle ++(0.23,0.23);
\fill[black!64] (3.79,1.68) rectangle ++(0.23,0.23);
\fill[black!64] (4.03,1.68) rectangle ++(0.23,0.23);
\fill[black!0] (4.27,1.68) rectangle ++(0.23,0.23);
\fill[black!48] (4.51,1.68) rectangle ++(0.23,0.23);
\fill[black!67] (4.75,1.68) rectangle ++(0.23,0.23);
\fill[black!56] (4.99,1.68) rectangle ++(0.23,0.23);
\fill[black!0] (5.23,1.68) rectangle ++(0.23,0.23);
\fill[black!64] (5.47,1.68) rectangle ++(0.23,0.23);
\fill[black!64] (5.71,1.68) rectangle ++(0.23,0.23);
\fill[black!51] (5.95,1.68) rectangle ++(0.23,0.23);
\fill[black!56] (6.19,1.68) rectangle ++(0.23,0.23);
\fill[black!67] (6.43,1.68) rectangle ++(0.23,0.23);
\fill[black!80] (6.67,1.68) rectangle ++(0.23,0.23);
\fill[black!67] (3.55,1.44) rectangle ++(0.23,0.23);
\fill[black!64] (3.79,1.44) rectangle ++(0.23,0.23);
\fill[black!64] (4.03,1.44) rectangle ++(0.23,0.23);
\fill[black!0] (4.27,1.44) rectangle ++(0.23,0.23);
\fill[black!43] (4.51,1.44) rectangle ++(0.23,0.23);
\fill[black!69] (4.75,1.44) rectangle ++(0.23,0.23);
\fill[black!45] (4.99,1.44) rectangle ++(0.23,0.23);
\fill[black!0] (5.23,1.44) rectangle ++(0.23,0.23);
\fill[black!64] (5.47,1.44) rectangle ++(0.23,0.23);
\fill[black!64] (5.71,1.44) rectangle ++(0.23,0.23);
\fill[black!51] (5.95,1.44) rectangle ++(0.23,0.23);
\fill[black!48] (6.19,1.44) rectangle ++(0.23,0.23);
\fill[black!64] (6.43,1.44) rectangle ++(0.23,0.23);
\fill[black!75] (6.67,1.44) rectangle ++(0.23,0.23);
\fill[black!64] (3.55,1.20) rectangle ++(0.23,0.23);
\fill[black!69] (3.79,1.20) rectangle ++(0.23,0.23);
\fill[black!69] (4.03,1.20) rectangle ++(0.23,0.23);
\fill[black!0] (4.27,1.20) rectangle ++(0.23,0.23);
\fill[black!40] (4.51,1.20) rectangle ++(0.23,0.23);
\fill[black!72] (4.75,1.20) rectangle ++(0.23,0.23);
\fill[black!45] (4.99,1.20) rectangle ++(0.23,0.23);
\fill[black!0] (5.23,1.20) rectangle ++(0.23,0.23);
\fill[black!64] (5.47,1.20) rectangle ++(0.23,0.23);
\fill[black!64] (5.71,1.20) rectangle ++(0.23,0.23);
\fill[black!53] (5.95,1.20) rectangle ++(0.23,0.23);
\fill[black!51] (6.19,1.20) rectangle ++(0.23,0.23);
\fill[black!59] (6.43,1.20) rectangle ++(0.23,0.23);
\fill[black!80] (6.67,1.20) rectangle ++(0.23,0.23);
\fill[black!64] (3.55,0.96) rectangle ++(0.23,0.23);
\fill[black!64] (3.79,0.96) rectangle ++(0.23,0.23);
\fill[black!61] (4.03,0.96) rectangle ++(0.23,0.23);
\fill[black!3] (4.27,0.96) rectangle ++(0.23,0.23);
\fill[black!48] (4.51,0.96) rectangle ++(0.23,0.23);
\fill[black!64] (4.75,0.96) rectangle ++(0.23,0.23);
\fill[black!53] (4.99,0.96) rectangle ++(0.23,0.23);
\fill[black!3] (5.23,0.96) rectangle ++(0.23,0.23);
\fill[black!64] (5.47,0.96) rectangle ++(0.23,0.23);
\fill[black!64] (5.71,0.96) rectangle ++(0.23,0.23);
\fill[black!43] (5.95,0.96) rectangle ++(0.23,0.23);
\fill[black!48] (6.19,0.96) rectangle ++(0.23,0.23);
\fill[black!69] (6.43,0.96) rectangle ++(0.23,0.23);
\fill[black!83] (6.67,0.96) rectangle ++(0.23,0.23);
\fill[black!69] (3.55,0.72) rectangle ++(0.23,0.23);
\fill[black!59] (3.79,0.72) rectangle ++(0.23,0.23);
\fill[black!64] (4.03,0.72) rectangle ++(0.23,0.23);
\fill[black!3] (4.27,0.72) rectangle ++(0.23,0.23);
\fill[black!48] (4.51,0.72) rectangle ++(0.23,0.23);
\fill[black!64] (4.75,0.72) rectangle ++(0.23,0.23);
\fill[black!48] (4.99,0.72) rectangle ++(0.23,0.23);
\fill[black!0] (5.23,0.72) rectangle ++(0.23,0.23);
\fill[black!64] (5.47,0.72) rectangle ++(0.23,0.23);
\fill[black!69] (5.71,0.72) rectangle ++(0.23,0.23);
\fill[black!48] (5.95,0.72) rectangle ++(0.23,0.23);
\fill[black!48] (6.19,0.72) rectangle ++(0.23,0.23);
\fill[black!67] (6.43,0.72) rectangle ++(0.23,0.23);
\fill[black!80] (6.67,0.72) rectangle ++(0.23,0.23);
\fill[black!64] (3.55,0.48) rectangle ++(0.23,0.23);
\fill[black!67] (3.79,0.48) rectangle ++(0.23,0.23);
\fill[black!64] (4.03,0.48) rectangle ++(0.23,0.23);
\fill[black!0] (4.27,0.48) rectangle ++(0.23,0.23);
\fill[black!48] (4.51,0.48) rectangle ++(0.23,0.23);
\fill[black!64] (4.75,0.48) rectangle ++(0.23,0.23);
\fill[black!48] (4.99,0.48) rectangle ++(0.23,0.23);
\fill[black!8] (5.23,0.48) rectangle ++(0.23,0.23);
\fill[black!64] (5.47,0.48) rectangle ++(0.23,0.23);
\fill[black!64] (5.71,0.48) rectangle ++(0.23,0.23);
\fill[black!48] (5.95,0.48) rectangle ++(0.23,0.23);
\fill[black!40] (6.19,0.48) rectangle ++(0.23,0.23);
\fill[black!56] (6.43,0.48) rectangle ++(0.23,0.23);
\fill[black!80] (6.67,0.48) rectangle ++(0.23,0.23);
\fill[black!64] (3.55,0.24) rectangle ++(0.23,0.23);
\fill[black!72] (3.79,0.24) rectangle ++(0.23,0.23);
\fill[black!64] (4.03,0.24) rectangle ++(0.23,0.23);
\fill[black!0] (4.27,0.24) rectangle ++(0.23,0.23);
\fill[black!40] (4.51,0.24) rectangle ++(0.23,0.23);
\fill[black!64] (4.75,0.24) rectangle ++(0.23,0.23);
\fill[black!51] (4.99,0.24) rectangle ++(0.23,0.23);
\fill[black!0] (5.23,0.24) rectangle ++(0.23,0.23);
\fill[black!64] (5.47,0.24) rectangle ++(0.23,0.23);
\fill[black!72] (5.71,0.24) rectangle ++(0.23,0.23);
\fill[black!48] (5.95,0.24) rectangle ++(0.23,0.23);
\fill[black!45] (6.19,0.24) rectangle ++(0.23,0.23);
\fill[black!56] (6.43,0.24) rectangle ++(0.23,0.23);
\fill[black!77] (6.67,0.24) rectangle ++(0.23,0.23);
\fill[black!61] (3.55,0.00) rectangle ++(0.23,0.23);
\fill[black!64] (3.79,0.00) rectangle ++(0.23,0.23);
\fill[black!64] (4.03,0.00) rectangle ++(0.23,0.23);
\fill[black!0] (4.27,0.00) rectangle ++(0.23,0.23);
\fill[black!48] (4.51,0.00) rectangle ++(0.23,0.23);
\fill[black!59] (4.75,0.00) rectangle ++(0.23,0.23);
\fill[black!43] (4.99,0.00) rectangle ++(0.23,0.23);
\fill[black!0] (5.23,0.00) rectangle ++(0.23,0.23);
\fill[black!64] (5.47,0.00) rectangle ++(0.23,0.23);
\fill[black!72] (5.71,0.00) rectangle ++(0.23,0.23);
\fill[black!48] (5.95,0.00) rectangle ++(0.23,0.23);
\fill[black!43] (6.19,0.00) rectangle ++(0.23,0.23);
\fill[black!64] (6.43,0.00) rectangle ++(0.23,0.23);
\fill[black!80] (6.67,0.00) rectangle ++(0.23,0.23);
\draw[gray!60] (3.54,-0.01) rectangle (6.91,3.36);
\node[anchor=south,font=\scriptsize\bfseries] at (8.78,3.78) {base rate};
\node[anchor=south,font=\scriptsize] at (8.78,3.41) {identical rows, no lift};
\fill[black!11] (7.10,3.12) rectangle ++(0.23,0.23);
\fill[black!0] (7.34,3.12) rectangle ++(0.23,0.23);
\fill[black!16] (7.58,3.12) rectangle ++(0.23,0.23);
\fill[black!0] (7.82,3.12) rectangle ++(0.23,0.23);
\fill[black!51] (8.06,3.12) rectangle ++(0.23,0.23);
\fill[black!32] (8.30,3.12) rectangle ++(0.23,0.23);
\fill[black!77] (8.54,3.12) rectangle ++(0.23,0.23);
\fill[black!29] (8.78,3.12) rectangle ++(0.23,0.23);
\fill[black!8] (9.02,3.12) rectangle ++(0.23,0.23);
\fill[black!24] (9.26,3.12) rectangle ++(0.23,0.23);
\fill[black!35] (9.50,3.12) rectangle ++(0.23,0.23);
\fill[black!43] (9.74,3.12) rectangle ++(0.23,0.23);
\fill[black!32] (9.98,3.12) rectangle ++(0.23,0.23);
\fill[black!80] (10.22,3.12) rectangle ++(0.23,0.23);
\fill[black!16] (7.10,2.88) rectangle ++(0.23,0.23);
\fill[black!3] (7.34,2.88) rectangle ++(0.23,0.23);
\fill[black!19] (7.58,2.88) rectangle ++(0.23,0.23);
\fill[black!0] (7.82,2.88) rectangle ++(0.23,0.23);
\fill[black!48] (8.06,2.88) rectangle ++(0.23,0.23);
\fill[black!32] (8.30,2.88) rectangle ++(0.23,0.23);
\fill[black!80] (8.54,2.88) rectangle ++(0.23,0.23);
\fill[black!27] (8.78,2.88) rectangle ++(0.23,0.23);
\fill[black!19] (9.02,2.88) rectangle ++(0.23,0.23);
\fill[black!29] (9.26,2.88) rectangle ++(0.23,0.23);
\fill[black!24] (9.50,2.88) rectangle ++(0.23,0.23);
\fill[black!53] (9.74,2.88) rectangle ++(0.23,0.23);
\fill[black!35] (9.98,2.88) rectangle ++(0.23,0.23);
\fill[black!83] (10.22,2.88) rectangle ++(0.23,0.23);
\fill[black!16] (7.10,2.64) rectangle ++(0.23,0.23);
\fill[black!3] (7.34,2.64) rectangle ++(0.23,0.23);
\fill[black!11] (7.58,2.64) rectangle ++(0.23,0.23);
\fill[black!0] (7.82,2.64) rectangle ++(0.23,0.23);
\fill[black!53] (8.06,2.64) rectangle ++(0.23,0.23);
\fill[black!32] (8.30,2.64) rectangle ++(0.23,0.23);
\fill[black!85] (8.54,2.64) rectangle ++(0.23,0.23);
\fill[black!37] (8.78,2.64) rectangle ++(0.23,0.23);
\fill[black!13] (9.02,2.64) rectangle ++(0.23,0.23);
\fill[black!32] (9.26,2.64) rectangle ++(0.23,0.23);
\fill[black!40] (9.50,2.64) rectangle ++(0.23,0.23);
\fill[black!51] (9.74,2.64) rectangle ++(0.23,0.23);
\fill[black!32] (9.98,2.64) rectangle ++(0.23,0.23);
\fill[black!80] (10.22,2.64) rectangle ++(0.23,0.23);
\fill[black!13] (7.10,2.40) rectangle ++(0.23,0.23);
\fill[black!0] (7.34,2.40) rectangle ++(0.23,0.23);
\fill[black!11] (7.58,2.40) rectangle ++(0.23,0.23);
\fill[black!0] (7.82,2.40) rectangle ++(0.23,0.23);
\fill[black!53] (8.06,2.40) rectangle ++(0.23,0.23);
\fill[black!40] (8.30,2.40) rectangle ++(0.23,0.23);
\fill[black!77] (8.54,2.40) rectangle ++(0.23,0.23);
\fill[black!37] (8.78,2.40) rectangle ++(0.23,0.23);
\fill[black!8] (9.02,2.40) rectangle ++(0.23,0.23);
\fill[black!32] (9.26,2.40) rectangle ++(0.23,0.23);
\fill[black!37] (9.50,2.40) rectangle ++(0.23,0.23);
\fill[black!48] (9.74,2.40) rectangle ++(0.23,0.23);
\fill[black!37] (9.98,2.40) rectangle ++(0.23,0.23);
\fill[black!77] (10.22,2.40) rectangle ++(0.23,0.23);
\fill[black!24] (7.10,2.16) rectangle ++(0.23,0.23);
\fill[black!0] (7.34,2.16) rectangle ++(0.23,0.23);
\fill[black!21] (7.58,2.16) rectangle ++(0.23,0.23);
\fill[black!8] (7.82,2.16) rectangle ++(0.23,0.23);
\fill[black!48] (8.06,2.16) rectangle ++(0.23,0.23);
\fill[black!35] (8.30,2.16) rectangle ++(0.23,0.23);
\fill[black!80] (8.54,2.16) rectangle ++(0.23,0.23);
\fill[black!24] (8.78,2.16) rectangle ++(0.23,0.23);
\fill[black!16] (9.02,2.16) rectangle ++(0.23,0.23);
\fill[black!32] (9.26,2.16) rectangle ++(0.23,0.23);
\fill[black!32] (9.50,2.16) rectangle ++(0.23,0.23);
\fill[black!40] (9.74,2.16) rectangle ++(0.23,0.23);
\fill[black!32] (9.98,2.16) rectangle ++(0.23,0.23);
\fill[black!75] (10.22,2.16) rectangle ++(0.23,0.23);
\fill[black!19] (7.10,1.92) rectangle ++(0.23,0.23);
\fill[black!0] (7.34,1.92) rectangle ++(0.23,0.23);
\fill[black!16] (7.58,1.92) rectangle ++(0.23,0.23);
\fill[black!0] (7.82,1.92) rectangle ++(0.23,0.23);
\fill[black!48] (8.06,1.92) rectangle ++(0.23,0.23);
\fill[black!32] (8.30,1.92) rectangle ++(0.23,0.23);
\fill[black!85] (8.54,1.92) rectangle ++(0.23,0.23);
\fill[black!32] (8.78,1.92) rectangle ++(0.23,0.23);
\fill[black!8] (9.02,1.92) rectangle ++(0.23,0.23);
\fill[black!37] (9.26,1.92) rectangle ++(0.23,0.23);
\fill[black!27] (9.50,1.92) rectangle ++(0.23,0.23);
\fill[black!51] (9.74,1.92) rectangle ++(0.23,0.23);
\fill[black!37] (9.98,1.92) rectangle ++(0.23,0.23);
\fill[black!88] (10.22,1.92) rectangle ++(0.23,0.23);
\fill[black!8] (7.10,1.68) rectangle ++(0.23,0.23);
\fill[black!0] (7.34,1.68) rectangle ++(0.23,0.23);
\fill[black!13] (7.58,1.68) rectangle ++(0.23,0.23);
\fill[black!0] (7.82,1.68) rectangle ++(0.23,0.23);
\fill[black!48] (8.06,1.68) rectangle ++(0.23,0.23);
\fill[black!37] (8.30,1.68) rectangle ++(0.23,0.23);
\fill[black!80] (8.54,1.68) rectangle ++(0.23,0.23);
\fill[black!32] (8.78,1.68) rectangle ++(0.23,0.23);
\fill[black!13] (9.02,1.68) rectangle ++(0.23,0.23);
\fill[black!32] (9.26,1.68) rectangle ++(0.23,0.23);
\fill[black!32] (9.50,1.68) rectangle ++(0.23,0.23);
\fill[black!48] (9.74,1.68) rectangle ++(0.23,0.23);
\fill[black!32] (9.98,1.68) rectangle ++(0.23,0.23);
\fill[black!88] (10.22,1.68) rectangle ++(0.23,0.23);
\fill[black!8] (7.10,1.44) rectangle ++(0.23,0.23);
\fill[black!0] (7.34,1.44) rectangle ++(0.23,0.23);
\fill[black!19] (7.58,1.44) rectangle ++(0.23,0.23);
\fill[black!8] (7.82,1.44) rectangle ++(0.23,0.23);
\fill[black!48] (8.06,1.44) rectangle ++(0.23,0.23);
\fill[black!35] (8.30,1.44) rectangle ++(0.23,0.23);
\fill[black!75] (8.54,1.44) rectangle ++(0.23,0.23);
\fill[black!27] (8.78,1.44) rectangle ++(0.23,0.23);
\fill[black!11] (9.02,1.44) rectangle ++(0.23,0.23);
\fill[black!29] (9.26,1.44) rectangle ++(0.23,0.23);
\fill[black!27] (9.50,1.44) rectangle ++(0.23,0.23);
\fill[black!48] (9.74,1.44) rectangle ++(0.23,0.23);
\fill[black!32] (9.98,1.44) rectangle ++(0.23,0.23);
\fill[black!80] (10.22,1.44) rectangle ++(0.23,0.23);
\fill[black!13] (7.10,1.20) rectangle ++(0.23,0.23);
\fill[black!0] (7.34,1.20) rectangle ++(0.23,0.23);
\fill[black!16] (7.58,1.20) rectangle ++(0.23,0.23);
\fill[black!0] (7.82,1.20) rectangle ++(0.23,0.23);
\fill[black!48] (8.06,1.20) rectangle ++(0.23,0.23);
\fill[black!32] (8.30,1.20) rectangle ++(0.23,0.23);
\fill[black!85] (8.54,1.20) rectangle ++(0.23,0.23);
\fill[black!29] (8.78,1.20) rectangle ++(0.23,0.23);
\fill[black!24] (9.02,1.20) rectangle ++(0.23,0.23);
\fill[black!32] (9.26,1.20) rectangle ++(0.23,0.23);
\fill[black!32] (9.50,1.20) rectangle ++(0.23,0.23);
\fill[black!48] (9.74,1.20) rectangle ++(0.23,0.23);
\fill[black!35] (9.98,1.20) rectangle ++(0.23,0.23);
\fill[black!80] (10.22,1.20) rectangle ++(0.23,0.23);
\fill[black!16] (7.10,0.96) rectangle ++(0.23,0.23);
\fill[black!3] (7.34,0.96) rectangle ++(0.23,0.23);
\fill[black!21] (7.58,0.96) rectangle ++(0.23,0.23);
\fill[black!8] (7.82,0.96) rectangle ++(0.23,0.23);
\fill[black!53] (8.06,0.96) rectangle ++(0.23,0.23);
\fill[black!32] (8.30,0.96) rectangle ++(0.23,0.23);
\fill[black!72] (8.54,0.96) rectangle ++(0.23,0.23);
\fill[black!29] (8.78,0.96) rectangle ++(0.23,0.23);
\fill[black!21] (9.02,0.96) rectangle ++(0.23,0.23);
\fill[black!35] (9.26,0.96) rectangle ++(0.23,0.23);
\fill[black!32] (9.50,0.96) rectangle ++(0.23,0.23);
\fill[black!40] (9.74,0.96) rectangle ++(0.23,0.23);
\fill[black!32] (9.98,0.96) rectangle ++(0.23,0.23);
\fill[black!83] (10.22,0.96) rectangle ++(0.23,0.23);
\fill[black!16] (7.10,0.72) rectangle ++(0.23,0.23);
\fill[black!0] (7.34,0.72) rectangle ++(0.23,0.23);
\fill[black!24] (7.58,0.72) rectangle ++(0.23,0.23);
\fill[black!0] (7.82,0.72) rectangle ++(0.23,0.23);
\fill[black!48] (8.06,0.72) rectangle ++(0.23,0.23);
\fill[black!35] (8.30,0.72) rectangle ++(0.23,0.23);
\fill[black!75] (8.54,0.72) rectangle ++(0.23,0.23);
\fill[black!32] (8.78,0.72) rectangle ++(0.23,0.23);
\fill[black!11] (9.02,0.72) rectangle ++(0.23,0.23);
\fill[black!37] (9.26,0.72) rectangle ++(0.23,0.23);
\fill[black!27] (9.50,0.72) rectangle ++(0.23,0.23);
\fill[black!45] (9.74,0.72) rectangle ++(0.23,0.23);
\fill[black!40] (9.98,0.72) rectangle ++(0.23,0.23);
\fill[black!83] (10.22,0.72) rectangle ++(0.23,0.23);
\fill[black!16] (7.10,0.48) rectangle ++(0.23,0.23);
\fill[black!0] (7.34,0.48) rectangle ++(0.23,0.23);
\fill[black!24] (7.58,0.48) rectangle ++(0.23,0.23);
\fill[black!0] (7.82,0.48) rectangle ++(0.23,0.23);
\fill[black!48] (8.06,0.48) rectangle ++(0.23,0.23);
\fill[black!24] (8.30,0.48) rectangle ++(0.23,0.23);
\fill[black!77] (8.54,0.48) rectangle ++(0.23,0.23);
\fill[black!35] (8.78,0.48) rectangle ++(0.23,0.23);
\fill[black!21] (9.02,0.48) rectangle ++(0.23,0.23);
\fill[black!27] (9.26,0.48) rectangle ++(0.23,0.23);
\fill[black!32] (9.50,0.48) rectangle ++(0.23,0.23);
\fill[black!43] (9.74,0.48) rectangle ++(0.23,0.23);
\fill[black!29] (9.98,0.48) rectangle ++(0.23,0.23);
\fill[black!85] (10.22,0.48) rectangle ++(0.23,0.23);
\fill[black!21] (7.10,0.24) rectangle ++(0.23,0.23);
\fill[black!0] (7.34,0.24) rectangle ++(0.23,0.23);
\fill[black!16] (7.58,0.24) rectangle ++(0.23,0.23);
\fill[black!0] (7.82,0.24) rectangle ++(0.23,0.23);
\fill[black!48] (8.06,0.24) rectangle ++(0.23,0.23);
\fill[black!24] (8.30,0.24) rectangle ++(0.23,0.23);
\fill[black!80] (8.54,0.24) rectangle ++(0.23,0.23);
\fill[black!32] (8.78,0.24) rectangle ++(0.23,0.23);
\fill[black!16] (9.02,0.24) rectangle ++(0.23,0.23);
\fill[black!29] (9.26,0.24) rectangle ++(0.23,0.23);
\fill[black!32] (9.50,0.24) rectangle ++(0.23,0.23);
\fill[black!45] (9.74,0.24) rectangle ++(0.23,0.23);
\fill[black!37] (9.98,0.24) rectangle ++(0.23,0.23);
\fill[black!80] (10.22,0.24) rectangle ++(0.23,0.23);
\fill[black!16] (7.10,0.00) rectangle ++(0.23,0.23);
\fill[black!8] (7.34,0.00) rectangle ++(0.23,0.23);
\fill[black!11] (7.58,0.00) rectangle ++(0.23,0.23);
\fill[black!0] (7.82,0.00) rectangle ++(0.23,0.23);
\fill[black!56] (8.06,0.00) rectangle ++(0.23,0.23);
\fill[black!27] (8.30,0.00) rectangle ++(0.23,0.23);
\fill[black!80] (8.54,0.00) rectangle ++(0.23,0.23);
\fill[black!32] (8.78,0.00) rectangle ++(0.23,0.23);
\fill[black!13] (9.02,0.00) rectangle ++(0.23,0.23);
\fill[black!32] (9.26,0.00) rectangle ++(0.23,0.23);
\fill[black!40] (9.50,0.00) rectangle ++(0.23,0.23);
\fill[black!53] (9.74,0.00) rectangle ++(0.23,0.23);
\fill[black!32] (9.98,0.00) rectangle ++(0.23,0.23);
\fill[black!75] (10.22,0.00) rectangle ++(0.23,0.23);
\draw[gray!60] (7.09,-0.01) rectangle (10.46,3.36);
\node[anchor=south,font=\scriptsize\bfseries] at (12.33,3.78) {Unmatched};
\node[anchor=south,font=\scriptsize] at (12.33,3.41) {rows differ, diagonal does not};
\fill[black!27] (10.65,3.12) rectangle ++(0.23,0.23);
\fill[black!53] (10.89,3.12) rectangle ++(0.23,0.23);
\fill[black!53] (11.13,3.12) rectangle ++(0.23,0.23);
\fill[black!27] (11.37,3.12) rectangle ++(0.23,0.23);
\fill[black!53] (11.61,3.12) rectangle ++(0.23,0.23);
\fill[black!67] (11.85,3.12) rectangle ++(0.23,0.23);
\fill[black!27] (12.09,3.12) rectangle ++(0.23,0.23);
\fill[black!40] (12.33,3.12) rectangle ++(0.23,0.23);
\fill[black!40] (12.57,3.12) rectangle ++(0.23,0.23);
\fill[black!40] (12.81,3.12) rectangle ++(0.23,0.23);
\fill[black!40] (13.05,3.12) rectangle ++(0.23,0.23);
\fill[black!27] (13.29,3.12) rectangle ++(0.23,0.23);
\fill[black!0] (13.53,3.12) rectangle ++(0.23,0.23);
\fill[black!40] (13.77,3.12) rectangle ++(0.23,0.23);
\fill[black!53] (10.65,2.88) rectangle ++(0.23,0.23);
\fill[black!40] (10.89,2.88) rectangle ++(0.23,0.23);
\fill[black!67] (11.13,2.88) rectangle ++(0.23,0.23);
\fill[black!40] (11.37,2.88) rectangle ++(0.23,0.23);
\fill[black!40] (11.61,2.88) rectangle ++(0.23,0.23);
\fill[black!40] (11.85,2.88) rectangle ++(0.23,0.23);
\fill[black!53] (12.09,2.88) rectangle ++(0.23,0.23);
\fill[black!40] (12.33,2.88) rectangle ++(0.23,0.23);
\fill[black!27] (12.57,2.88) rectangle ++(0.23,0.23);
\fill[black!27] (12.81,2.88) rectangle ++(0.23,0.23);
\fill[black!53] (13.05,2.88) rectangle ++(0.23,0.23);
\fill[black!53] (13.29,2.88) rectangle ++(0.23,0.23);
\fill[black!0] (13.53,2.88) rectangle ++(0.23,0.23);
\fill[black!53] (13.77,2.88) rectangle ++(0.23,0.23);
\fill[black!40] (10.65,2.64) rectangle ++(0.23,0.23);
\fill[black!40] (10.89,2.64) rectangle ++(0.23,0.23);
\fill[black!53] (11.13,2.64) rectangle ++(0.23,0.23);
\fill[black!40] (11.37,2.64) rectangle ++(0.23,0.23);
\fill[black!53] (11.61,2.64) rectangle ++(0.23,0.23);
\fill[black!53] (11.85,2.64) rectangle ++(0.23,0.23);
\fill[black!40] (12.09,2.64) rectangle ++(0.23,0.23);
\fill[black!40] (12.33,2.64) rectangle ++(0.23,0.23);
\fill[black!40] (12.57,2.64) rectangle ++(0.23,0.23);
\fill[black!27] (12.81,2.64) rectangle ++(0.23,0.23);
\fill[black!27] (13.05,2.64) rectangle ++(0.23,0.23);
\fill[black!40] (13.29,2.64) rectangle ++(0.23,0.23);
\fill[black!0] (13.53,2.64) rectangle ++(0.23,0.23);
\fill[black!40] (13.77,2.64) rectangle ++(0.23,0.23);
\fill[black!53] (10.65,2.40) rectangle ++(0.23,0.23);
\fill[black!40] (10.89,2.40) rectangle ++(0.23,0.23);
\fill[black!67] (11.13,2.40) rectangle ++(0.23,0.23);
\fill[black!40] (11.37,2.40) rectangle ++(0.23,0.23);
\fill[black!40] (11.61,2.40) rectangle ++(0.23,0.23);
\fill[black!40] (11.85,2.40) rectangle ++(0.23,0.23);
\fill[black!53] (12.09,2.40) rectangle ++(0.23,0.23);
\fill[black!40] (12.33,2.40) rectangle ++(0.23,0.23);
\fill[black!27] (12.57,2.40) rectangle ++(0.23,0.23);
\fill[black!27] (12.81,2.40) rectangle ++(0.23,0.23);
\fill[black!53] (13.05,2.40) rectangle ++(0.23,0.23);
\fill[black!53] (13.29,2.40) rectangle ++(0.23,0.23);
\fill[black!0] (13.53,2.40) rectangle ++(0.23,0.23);
\fill[black!53] (13.77,2.40) rectangle ++(0.23,0.23);
\fill[black!0] (10.65,2.16) rectangle ++(0.23,0.23);
\fill[black!53] (10.89,2.16) rectangle ++(0.23,0.23);
\fill[black!80] (11.13,2.16) rectangle ++(0.23,0.23);
\fill[black!27] (11.37,2.16) rectangle ++(0.23,0.23);
\fill[black!67] (11.61,2.16) rectangle ++(0.23,0.23);
\fill[black!53] (11.85,2.16) rectangle ++(0.23,0.23);
\fill[black!40] (12.09,2.16) rectangle ++(0.23,0.23);
\fill[black!67] (12.33,2.16) rectangle ++(0.23,0.23);
\fill[black!53] (12.57,2.16) rectangle ++(0.23,0.23);
\fill[black!27] (12.81,2.16) rectangle ++(0.23,0.23);
\fill[black!27] (13.05,2.16) rectangle ++(0.23,0.23);
\fill[black!40] (13.29,2.16) rectangle ++(0.23,0.23);
\fill[black!0] (13.53,2.16) rectangle ++(0.23,0.23);
\fill[black!40] (13.77,2.16) rectangle ++(0.23,0.23);
\fill[black!53] (10.65,1.92) rectangle ++(0.23,0.23);
\fill[black!40] (10.89,1.92) rectangle ++(0.23,0.23);
\fill[black!67] (11.13,1.92) rectangle ++(0.23,0.23);
\fill[black!40] (11.37,1.92) rectangle ++(0.23,0.23);
\fill[black!40] (11.61,1.92) rectangle ++(0.23,0.23);
\fill[black!40] (11.85,1.92) rectangle ++(0.23,0.23);
\fill[black!53] (12.09,1.92) rectangle ++(0.23,0.23);
\fill[black!40] (12.33,1.92) rectangle ++(0.23,0.23);
\fill[black!27] (12.57,1.92) rectangle ++(0.23,0.23);
\fill[black!27] (12.81,1.92) rectangle ++(0.23,0.23);
\fill[black!53] (13.05,1.92) rectangle ++(0.23,0.23);
\fill[black!53] (13.29,1.92) rectangle ++(0.23,0.23);
\fill[black!0] (13.53,1.92) rectangle ++(0.23,0.23);
\fill[black!53] (13.77,1.92) rectangle ++(0.23,0.23);
\fill[black!40] (10.65,1.68) rectangle ++(0.23,0.23);
\fill[black!40] (10.89,1.68) rectangle ++(0.23,0.23);
\fill[black!53] (11.13,1.68) rectangle ++(0.23,0.23);
\fill[black!40] (11.37,1.68) rectangle ++(0.23,0.23);
\fill[black!53] (11.61,1.68) rectangle ++(0.23,0.23);
\fill[black!53] (11.85,1.68) rectangle ++(0.23,0.23);
\fill[black!40] (12.09,1.68) rectangle ++(0.23,0.23);
\fill[black!40] (12.33,1.68) rectangle ++(0.23,0.23);
\fill[black!40] (12.57,1.68) rectangle ++(0.23,0.23);
\fill[black!27] (12.81,1.68) rectangle ++(0.23,0.23);
\fill[black!27] (13.05,1.68) rectangle ++(0.23,0.23);
\fill[black!40] (13.29,1.68) rectangle ++(0.23,0.23);
\fill[black!0] (13.53,1.68) rectangle ++(0.23,0.23);
\fill[black!40] (13.77,1.68) rectangle ++(0.23,0.23);
\fill[black!53] (10.65,1.44) rectangle ++(0.23,0.23);
\fill[black!40] (10.89,1.44) rectangle ++(0.23,0.23);
\fill[black!67] (11.13,1.44) rectangle ++(0.23,0.23);
\fill[black!40] (11.37,1.44) rectangle ++(0.23,0.23);
\fill[black!40] (11.61,1.44) rectangle ++(0.23,0.23);
\fill[black!40] (11.85,1.44) rectangle ++(0.23,0.23);
\fill[black!53] (12.09,1.44) rectangle ++(0.23,0.23);
\fill[black!40] (12.33,1.44) rectangle ++(0.23,0.23);
\fill[black!27] (12.57,1.44) rectangle ++(0.23,0.23);
\fill[black!27] (12.81,1.44) rectangle ++(0.23,0.23);
\fill[black!53] (13.05,1.44) rectangle ++(0.23,0.23);
\fill[black!53] (13.29,1.44) rectangle ++(0.23,0.23);
\fill[black!0] (13.53,1.44) rectangle ++(0.23,0.23);
\fill[black!53] (13.77,1.44) rectangle ++(0.23,0.23);
\fill[black!27] (10.65,1.20) rectangle ++(0.23,0.23);
\fill[black!53] (10.89,1.20) rectangle ++(0.23,0.23);
\fill[black!53] (11.13,1.20) rectangle ++(0.23,0.23);
\fill[black!27] (11.37,1.20) rectangle ++(0.23,0.23);
\fill[black!53] (11.61,1.20) rectangle ++(0.23,0.23);
\fill[black!67] (11.85,1.20) rectangle ++(0.23,0.23);
\fill[black!27] (12.09,1.20) rectangle ++(0.23,0.23);
\fill[black!40] (12.33,1.20) rectangle ++(0.23,0.23);
\fill[black!40] (12.57,1.20) rectangle ++(0.23,0.23);
\fill[black!40] (12.81,1.20) rectangle ++(0.23,0.23);
\fill[black!40] (13.05,1.20) rectangle ++(0.23,0.23);
\fill[black!27] (13.29,1.20) rectangle ++(0.23,0.23);
\fill[black!0] (13.53,1.20) rectangle ++(0.23,0.23);
\fill[black!40] (13.77,1.20) rectangle ++(0.23,0.23);
\fill[black!27] (10.65,0.96) rectangle ++(0.23,0.23);
\fill[black!53] (10.89,0.96) rectangle ++(0.23,0.23);
\fill[black!53] (11.13,0.96) rectangle ++(0.23,0.23);
\fill[black!27] (11.37,0.96) rectangle ++(0.23,0.23);
\fill[black!53] (11.61,0.96) rectangle ++(0.23,0.23);
\fill[black!67] (11.85,0.96) rectangle ++(0.23,0.23);
\fill[black!27] (12.09,0.96) rectangle ++(0.23,0.23);
\fill[black!40] (12.33,0.96) rectangle ++(0.23,0.23);
\fill[black!40] (12.57,0.96) rectangle ++(0.23,0.23);
\fill[black!40] (12.81,0.96) rectangle ++(0.23,0.23);
\fill[black!40] (13.05,0.96) rectangle ++(0.23,0.23);
\fill[black!27] (13.29,0.96) rectangle ++(0.23,0.23);
\fill[black!0] (13.53,0.96) rectangle ++(0.23,0.23);
\fill[black!40] (13.77,0.96) rectangle ++(0.23,0.23);
\fill[black!67] (10.65,0.72) rectangle ++(0.23,0.23);
\fill[black!40] (10.89,0.72) rectangle ++(0.23,0.23);
\fill[black!40] (11.13,0.72) rectangle ++(0.23,0.23);
\fill[black!40] (11.37,0.72) rectangle ++(0.23,0.23);
\fill[black!27] (11.61,0.72) rectangle ++(0.23,0.23);
\fill[black!13] (11.85,0.72) rectangle ++(0.23,0.23);
\fill[black!53] (12.09,0.72) rectangle ++(0.23,0.23);
\fill[black!67] (12.33,0.72) rectangle ++(0.23,0.23);
\fill[black!27] (12.57,0.72) rectangle ++(0.23,0.23);
\fill[black!27] (12.81,0.72) rectangle ++(0.23,0.23);
\fill[black!27] (13.05,0.72) rectangle ++(0.23,0.23);
\fill[black!27] (13.29,0.72) rectangle ++(0.23,0.23);
\fill[black!0] (13.53,0.72) rectangle ++(0.23,0.23);
\fill[black!67] (13.77,0.72) rectangle ++(0.23,0.23);
\fill[black!0] (10.65,0.48) rectangle ++(0.23,0.23);
\fill[black!53] (10.89,0.48) rectangle ++(0.23,0.23);
\fill[black!80] (11.13,0.48) rectangle ++(0.23,0.23);
\fill[black!27] (11.37,0.48) rectangle ++(0.23,0.23);
\fill[black!67] (11.61,0.48) rectangle ++(0.23,0.23);
\fill[black!53] (11.85,0.48) rectangle ++(0.23,0.23);
\fill[black!40] (12.09,0.48) rectangle ++(0.23,0.23);
\fill[black!67] (12.33,0.48) rectangle ++(0.23,0.23);
\fill[black!53] (12.57,0.48) rectangle ++(0.23,0.23);
\fill[black!27] (12.81,0.48) rectangle ++(0.23,0.23);
\fill[black!27] (13.05,0.48) rectangle ++(0.23,0.23);
\fill[black!40] (13.29,0.48) rectangle ++(0.23,0.23);
\fill[black!0] (13.53,0.48) rectangle ++(0.23,0.23);
\fill[black!40] (13.77,0.48) rectangle ++(0.23,0.23);
\fill[black!67] (10.65,0.24) rectangle ++(0.23,0.23);
\fill[black!40] (10.89,0.24) rectangle ++(0.23,0.23);
\fill[black!40] (11.13,0.24) rectangle ++(0.23,0.23);
\fill[black!40] (11.37,0.24) rectangle ++(0.23,0.23);
\fill[black!27] (11.61,0.24) rectangle ++(0.23,0.23);
\fill[black!13] (11.85,0.24) rectangle ++(0.23,0.23);
\fill[black!53] (12.09,0.24) rectangle ++(0.23,0.23);
\fill[black!67] (12.33,0.24) rectangle ++(0.23,0.23);
\fill[black!27] (12.57,0.24) rectangle ++(0.23,0.23);
\fill[black!27] (12.81,0.24) rectangle ++(0.23,0.23);
\fill[black!27] (13.05,0.24) rectangle ++(0.23,0.23);
\fill[black!27] (13.29,0.24) rectangle ++(0.23,0.23);
\fill[black!0] (13.53,0.24) rectangle ++(0.23,0.23);
\fill[black!67] (13.77,0.24) rectangle ++(0.23,0.23);
\fill[black!67] (10.65,0.00) rectangle ++(0.23,0.23);
\fill[black!40] (10.89,0.00) rectangle ++(0.23,0.23);
\fill[black!40] (11.13,0.00) rectangle ++(0.23,0.23);
\fill[black!40] (11.37,0.00) rectangle ++(0.23,0.23);
\fill[black!27] (11.61,0.00) rectangle ++(0.23,0.23);
\fill[black!13] (11.85,0.00) rectangle ++(0.23,0.23);
\fill[black!53] (12.09,0.00) rectangle ++(0.23,0.23);
\fill[black!67] (12.33,0.00) rectangle ++(0.23,0.23);
\fill[black!27] (12.57,0.00) rectangle ++(0.23,0.23);
\fill[black!27] (12.81,0.00) rectangle ++(0.23,0.23);
\fill[black!27] (13.05,0.00) rectangle ++(0.23,0.23);
\fill[black!27] (13.29,0.00) rectangle ++(0.23,0.23);
\fill[black!0] (13.53,0.00) rectangle ++(0.23,0.23);
\fill[black!67] (13.77,0.00) rectangle ++(0.23,0.23);
\draw[gray!60] (10.64,-0.01) rectangle (14.01,3.36);
\end{tikzpicture}
\caption{The cross-prediction matrix $M[i][j]=r_i(y_j)$ under the four regimes, on $14$ in-scope
units; darkness is the score, normalised per panel, and all cells may carry noise. The regimes are defined in \cref{tab:classes}.
\emph{Matched}: each unit's own restriction fits its own outcome, so only the diagonal is
dark. \emph{Shared}: one restriction over the whole scope---identical rows, and in scope real
lift. \emph{Base rate}: identical rows too, but satisfied as often outside the scope, so no
lift. \emph{Unmatched}: restrictions differ, none matched to its own outcome, so no diagonal
stands out.}
\label{fig:matrix}
\end{figure}

\subsection{Certification}\label{sec:decision}\label{sec:global}\label{sec:theory}
\begin{proposition}[Significance validity]\label{prop:sig}
$\Ex[\ejsig]=1+\Ex[r_j(y_j)]-\Ex[r_j(y_v)]$, so under $H_0^{\mathrm{sig}}$, $\Ex[\ejsig]\le1$: $\ejsig$ is an e-value for no lift, however else in-scope and out-of-scope outcomes are distributed.
\end{proposition}

\begin{proposition}[Uniqueness validity, conformal]\label{prop:mech}
$\eijuniq$ is an exact e-value for the exchangeability null $H_0^{\mathrm{uniq}}$: under a \emph{random} (no-mechanism) hypothesis $\Ex[\eijuniq]=1$, so only genuine per-unit structure drives it above $1$.
\end{proposition}

Both nulls are statements about the scope, not about one unit: $H_0^{\mathrm{sig}}$ compares means over
its units and $H_0^{\mathrm{uniq}}$ is a single claim about the pairing. That is the right shape,
because the verdict is a claim about the scope---and a hypothesis true of some units and false of
others is not what \emph{certify} asserts, but what \emph{Sparse} exists to name. The terms are combined by
betting rather than by their mean, which merges e-values under any dependence but settles at the average term,
and a mean of terms bounded by $2$ never clears $1/\alpha$. For terms $e_1,\dots,e_n$,
\begin{equation}\label{eq:agg}
\mathrm{bet}(e_1,\dots,e_n)\;=\;\frac{1}{|\Lambda|}\sum_{\lambda\in\Lambda}\ \prod_{t=1}^{n}\big(1+\lambda(e_t-1)\big),
\qquad \Lambda=\{0.05,0.1,0.2,0.3,\dots,0.9\},
\end{equation}
the wealth of a bettor who stakes a fraction $\lambda$ of it on every term, averaged over a grid of stakes
\citep{shafer2021testing,waudbysmith2024estimating}: $\Esig$ bets on the $|U|$ terms $\ejsig$ and $\Euniq$ on the
$\lfloor|U|/2\rfloor$ terms $\eijuniq$. The grid is fixed in advance, so no stake is tuned on the data, and a product
does not depend on the order of its terms.
\begin{proposition}[Betting aggregation]\label{prop:agg}
If the terms are independent and each has mean at most $1$, $\mathrm{bet}$ is an e-value. Hence $\Esig$ is an e-value
under $H_0^{\mathrm{sig}}$ for independently drawn units (\cref{app:theory}), and $\Euniq$ is an exact e-value under
$H_0^{\mathrm{uniq}}$ whatever the dependence between units.
\end{proposition}

\cref{tab:classes} gives the decision rule in full.
\begin{theorem}[Certification error]\label{thm:iut}
Every certificate requires $\Esig\ge1/\alpha$, so under $H_0^{\mathrm{sig}}$, for independently drawn units, the
per-hypothesis probability of certifying is $\le\alpha$. A \emph{Matched} certificate also requires
$\Euniq\ge1/\alpha$, and its probability under $H_0^{\mathrm{sig}}\cup H_0^{\mathrm{uniq}}$ is
$\le\alpha$. A \emph{Shared} certificate also asserts that one restriction serves the scope, an
attribution the invariance call bounds separately, at level $\alpha$ as $|U|$ grows (\cref{sec:invariance}).
Every certificate also asserts coverage, and under $H_0^{\mathrm{cov}}$---at most a fraction
$\tau_{\mathrm{cov}}$ of in-scope units have a gain, being likelier to outscore a random out-of-scope unit than to be
outscored---its probability is $\le\alpha$ given independent units (\cref{app:choosecov}); a scope that does not show
coverage is withheld as \emph{Sparse}.
\end{theorem}

A scope too small for $\mathrm{bet}$ to reach $1/\alpha$ on any data is \emph{undecided}: every term is at most $2$,
so at $\alpha=0.05$ significance needs $7$ units and uniqueness, whose terms are pairs, $14$ (\cref{app:alpha}); past
that, the wealth grows geometrically with $|U|$ whenever the mean term is clearly above $1$. At that level we suggest
$|U|=400$, with $100$ the minimum; a larger scope is subsampled uniformly, to at most $|T\setminus U|$ units so that each keeps a
reference, its reference still drawn from outside the whole scope. $\Esig$ is asked of every certificate and carries
its $\alpha$ bound; $\Euniq$ decides which certificate a hypothesis can earn; the certified share gates both, the
invariance call also gates \emph{Shared}, and $\gamma$ can only reject.

\begin{table}[htbp]
\caption{Every value combination of the four statistics with its verdict, for a scope covering at
most a fraction $\gamma$ of $T$; a wider one is \emph{base rate} before any statistic is read.
The invariance call is consulted only where $\Euniq<1/\alpha$: under the declared noise ceiling $k$ it is \emph{shared} when the restriction spread is shown to lie below the squared ceiling, and \emph{not shared} otherwise, including when the data cannot place it. ``---'' marks a statistic the verdict does not consult, the columns to its left having settled it.}
\label{tab:classes}
\centering
\scriptsize
\begin{tabularx}{\linewidth}{@{}cccc>{\raggedright\arraybackslash}p{2.15cm}X@{}}
\toprule
$\Esig$ & $\Euniq$ & call & $\hat c$ & Type / verdict & In plain terms \\
\midrule
$<1/\alpha$ & $<1/\alpha$ & shared & --- & Base rate \newline \emph{reject} & Same claim for everyone, and it fits in scope no better than outside. Uninformative rather than false. \\
$<1/\alpha$ & $<1/\alpha$ & not shared & --- & Unsupported \newline \emph{reject} & Not shown to be one claim, and it predicts nothing. \\
$<1/\alpha$ & $\ge1/\alpha$ & --- & --- & Subthreshold-Matched \newline \emph{reject} & Fits its own unit better than others, yet in scope no better than outside. \\
$\ge1/\alpha$ & $<1/\alpha$ & not shared & --- & Unmatched \newline \emph{reject} & Not shown to be one claim; it does predict, but shuffling who gets which changes nothing. \\
$\ge1/\alpha$ & $<1/\alpha$ & shared & $\ge\tau_{\mathrm{cov}}$ & Shared \newline \emph{certify} & One claim covering the scope, and it holds there. Real, but about the scope rather than the unit. \\
$\ge1/\alpha$ & $<1/\alpha$ & shared & $<\tau_{\mathrm{cov}}$ & Sparse \newline \emph{reject} & One claim, but its lift is not shown to reach enough of the units. \\
$\ge1/\alpha$ & $\ge1/\alpha$ & --- & $<\tau_{\mathrm{cov}}$ & Sparse \newline \emph{reject} & Per-unit prediction, not shown to reach enough of the units---the aggregate may say the rule holds where most units say it does not. \\
$\ge1/\alpha$ & $\ge1/\alpha$ & --- & $\ge\tau_{\mathrm{cov}}$ & Matched \newline \emph{certify} & The restriction carries information the scope alone does not: shuffle the assignment and prediction breaks. \\
\bottomrule
\end{tabularx}
\end{table}

\paragraph{Shared and base rate.}
A \emph{unit-invariant} hypothesis is one whose $\mathrm{apply}(x)$ is constant over its scope, where both $\mathrm{apply}$ and the restriction it generates may be non-deterministic. If its lift clears $1/\alpha$ it is \emph{Shared}, or \emph{Sparse} if $\hat c<\tau_{\mathrm{cov}}$; if not, or if its scope covers more than a fraction $\gamma$ of $T$, it is \emph{base rate}. $\gamma$ is a user input, $0.9$ by default, marking where a claim about a scope becomes one about the population (\cref{app:gamma}). Such a hypothesis is not mechanism-free: its mechanism is scope
membership itself. For a phone-buyer scope, an $\mathrm{apply}$ returning ``buys a case'' hands every
in-scope unit the \emph{same} restriction, so $\Euniq=1$ exactly and every column of $M$ is
constant: the spread vanishes, and the call is \emph{shared} whatever the ceiling, provided $\bar M$ is clear of zero (\cref{app:unitinv});
returning ``buys a case for this \emph{specific} phone'' makes it \emph{Matched}.

\paragraph{Cost.}\label{sec:cost} $\mathrm{apply}(x)$---normally the expensive step---runs once per in-scope unit
($O(|U|)$ calls); the rest is cheap scoring, $O(|U|^2)$ for the matrix and the invariance call and $O(|V|)$ for the
significance terms, and the only pass over all of $T$ is the scope test. The cost that does not scale is human: a
per-hypothesis resampling model must be \emph{written}, by someone deciding which covariates define scope membership
and must be held fixed and which carry the mechanism and must be redrawn---get that backwards and the test is invalid.
Ours needs no such model, at any number of hypotheses; what a batch of them costs in error control is \cref{app:ebh}.

\section{Distinguishing unit-invariant hypotheses}\label{sec:invariance}
Two pairs of verdicts turn on one question: over its scope, is the hypothesis
\textbf{invariant}---all units under one restriction, within a noise tolerance---or \textbf{variant}? \emph{Shared} against
\emph{Unmatched} asks it where the lift clears $1/\alpha$; \emph{base rate} against
\emph{Unsupported}, where it does not. A variant hypothesis's restrictions differ,
sometimes only slightly, since restrictions can differ yet share one hidden factor (\cref{fig:matrix}).

In the datasets we target, units are \textbf{i.i.d.}\ (\cref{app:theory}), so for distinct units $i\ne j$ the
restriction $r_i$ and the outcome $y_j$ are independent: the off-diagonal cells of $M$ are draws from the product of the restriction and
outcome laws, and the diagonal, where a unit meets its own outcome, carries the matching signal
$\Euniq$ reads. Under that product law each score splits as
\begin{equation}\label{eq:invmodel}
M[i][j]=\mu+a_i+b_j+c_{ij},
\end{equation}
with $\mu$ the mean score, $a_i$ how highly restriction $i$ scores overall, $b_j$ how easily outcome $j$ is fitted, and
$c_{ij}$ which restriction fits which outcome beyond both; $c$ averages to zero over either argument,
so the three terms are uncorrelated, with variances $\sigma_{\mathrm r}^2$, $\sigma_{\mathrm y}^2$ and
$\sigma_{\mathrm{ry}}^2$. Any randomness in how $\mathrm{apply}$ produces a restriction belongs to that
restriction, so no separate cell noise enters. What separates one shared restriction from several is
how differently the restrictions score a common outcome,
\begin{equation}\label{eq:sr}
S_{\mathrm r}=\tfrac12\,\Ex\big[(M[i][j]-M[k][j])^2\big]=\sigma_{\mathrm r}^2+\sigma_{\mathrm{ry}}^2,
\qquad i,k,j\ \text{distinct},
\end{equation}
which is $0$ for one restriction and grows with every way restrictions differ, in level or in
pattern (\cref{app:einvderiv}).

\textbf{The call.} The declared ceiling puts sharedness at $S_{\mathrm r}<(k\mu)^2$.
$\widehat S_{\mathrm r}$, the mean over columns of the variance of each column's off-diagonal cells,
averages $\tfrac12(M[i][j]-M[k][j])^2$ over every distinct triple and so estimates $S_{\mathrm r}$
without bias, and $\bar M$, the mean off-diagonal cell, estimates $\mu$. The call tests
$\widehat D=\widehat S_{\mathrm r}-(k\bar M)^2$, with standard error $\widehat{\mathrm{se}}$ from the
delete-one-unit jackknife of $\widehat D$ itself---a unit removing its row and its column from both
terms, so the ceiling's estimate carries its own error. With $z=\Phi^{-1}(1-\alpha)$,
\begin{equation}\label{eq:call}
\text{call}=\begin{cases}
\text{shared}, & \widehat D+z\,\widehat{\mathrm{se}}<0,\\
\text{variant}, & \widehat D-z\,\widehat{\mathrm{se}}>0,\\
\text{undecided}, & \text{otherwise},
\end{cases}
\end{equation}
and only \emph{shared} admits \emph{Shared} or \emph{base rate}. It is an equivalence test:
sharedness is claimed only when the data place the spread below the ceiling, and a scope they cannot
place is declined.

\textbf{What it bounds.} If the restrictions truly spread by the ceiling or more, the call is
\emph{shared} with probability at most $\alpha$; if they spread by less, it is \emph{variant} with
probability at most $\alpha$ (\cref{app:einvderiv}). Both bounds are asymptotic in $|U|$, through the
normal approximation to $\widehat D$'s error, and need units drawn independently and the regularity conditions of
\cref{app:einvderiv}; nothing else is assumed of the scores' law or of independence between cells. Only the first feeds a certificate. What $k$ decides is where one
restriction ends: set too high, restrictions that genuinely differ fall under it and an
\emph{Unmatched} scope certifies as \emph{Shared}; set too low, a shared restriction whose per-unit
variation exceeds it is declined.

\section{Experiments}\label{sec:exp}

\subsection{Setup}\label{sec:exp:setup}
\textbf{A verdict can only be scored against a known class, and no real corpus carries
one.} We therefore plant the classes, each defined by how its data are generated, not by our
statistics (\cref{tab:hyps}), and since planting only targets a class, experts audit every planted
sample against that definition; the audited class is the label. We evaluate at three
tiers. First, a generator we write:
both the world and the mechanism are ours. Second, an
\emph{intervention} on Amazon Video-Games---scope membership drawn at random and a mechanism
planted, so the world is real and only the mechanism is ours. Third, because the planted hypotheses are
still ours, $285$ hypotheses that language-model sessions wrote over Amazon Electronics knowing neither our method nor
our classes, each labeled by human experts (\cref{app:indep}), so neither the world nor the hypotheses are ours.

Four of the six types plant the documented confusions of \cref{sec:intro}: \emph{base rate} the
restated base rate and popularity bias, \emph{Shared} and \emph{Unmatched} the scope-level effect---carried by one
restriction, or by restrictions that differ from unit to unit---and \emph{Sparse} the subgroup-carried lift.
Exhausting the space is available to no dataset, so the reachable standard is coverage of the common
cases. All six appear in every tier, and the language-model tier's labels also include the seventh verdict,
\emph{Subthreshold-Matched}; on the generator, three types also carry instances from generators written from
scratch (\cref{tab:extrahyps}). Within each tier, every method is scored on the same samples. Neither planted tier was
built with the betting combination or the invariance call in view, and the language-model tier's hypotheses were
written by sessions that knew neither.

\textbf{The populations.} Throughout, $\alpha=0.05$, so an e-value must reach $1/\alpha=20$ to certify. The planted
tiers read $|U|=400$ units of each scope and $|V|=1000$ from outside it, both drawn uniformly; the generator gives $300$
samples, $50$ per type over seventeen rules, and the intervention $108$, $18$ per type. The language-model tier draws
its $|U|=400$ units and $|V|=2000$ references. Every count is over samples; generator, certifier and hypotheses are in
\cref{app:exp}, \cref{tab:hyps} and \cref{tab:extrahyps}, and the language-model tier's corpus, task and labels in
\cref{app:indep}.

\subsection{Results}\label{sec:exp:results}
We score each method's certify decision against the target its audited type implies---certify \emph{Matched} or
\emph{Shared}, reject the rest---per rule (\cref{tab:exp}), per type (\cref{tab:cover}), on the intervention tier
(\cref{tab:interv}) and on the language-model tier (\cref{tab:indep}). The
composite puts existing tests in place of three of our statistics---a randomization test of scope membership against
out-of-scope outcomes for the gain, the model-X CRT for the unique gain and a finite-sample betting bound for the
invariance call---uses the certified share for coverage, as we do, and routes the four through \cref{tab:classes}
(\cref{app:crtgates}); since it returns our verdicts, both are scored on them. The exact
permutation test is an ablation: it reads our $M$ within the scope, so it differs from $N\times N$ only in the
statistics taken off it.

$N\times N$ needs no per-hypothesis input: \cref{alg:certify} takes the hypothesis, the
data and four constants set once per deployment. The composite's unique-gain test needs a resampling law for the
covariates a hypothesis reads; we wrote six for the planted tiers, known only because we built the data, and
language-model sessions wrote one for each of the $285$ language-model hypotheses (\cref{app:indep}), $291$ in all.

\begin{table}[!htbp]
\caption{Per-rule certification over the $300$ generator samples. The composite and ours also name a verdict
(\cref{tab:classes}), and the last two columns count the samples whose verdict matches the planted type. The e-values
are medians on a log scale, so certifying needs $\log_{10}E\ge\log_{10}20=1.3$. \emph{call} is the invariance call at
$k=0.3$: \emph{shared}, \emph{variant} or \emph{undecided}; $\hat c$ is the median certified share.}
\label{tab:exp}
\scriptsize
\centering
{\setlength{\tabcolsep}{1.8pt}
\begin{tabular}{@{}lrrlrlccccc@{}}
\toprule
Rule & $\log_{10}\Esig$ & $\log_{10}\Euniq$ & call & $\hat c$ & our verdict & \multicolumn{3}{c}{matches target} & \multicolumn{2}{c}{verdict matches} \\
\cmidrule(lr){7-9}\cmidrule(l){10-11}
 & & & & & & ours & composite & held-out & ours & composite \\
\midrule
brand & 87.0 & 52.9 & variant & 0.85 & \emph{Matched} & $10/10$ & $10/10$ & $10/10$ & $10/10$ & $10/10$ \\
category & 88.8 & 52.3 & variant & 0.85 & \emph{Matched} & $10/10$ & $10/10$ & $10/10$ & $10/10$ & $10/10$ \\
complement & 60.1 & 53.8 & variant & 0.85 & \emph{Matched} & $10/10$ & $10/10$ & $10/10$ & $10/10$ & $10/10$ \\
complement\_cat & 60.1 & 53.4 & variant & 0.86 & \emph{Matched} & $10/10$ & $10/10$ & $10/10$ & $10/10$ & $10/10$ \\
latent readout & 94.2 & 51.7 & variant & 0.86 & \emph{Matched} & $10/10$ & $10/10$ & $10/10$ & $10/10$ & $10/10$ \\
constant, 3/50/80\% & 92.6 & 0.0 & shared & 0.85 & \emph{Shared} & $37/37$ & $37/37$ & $37/37$ & $37/37$ & $37/37$ \\
seasonal & 69.3 & 0.0 & shared & 0.85 & \emph{Shared} & $13/13$ & $0/13$ & $13/13$ & $13/13$ & $0/13$ \\
heavy\_pop & $-0.6$ & 0.0 & shared & 0.00 & \emph{base rate} & $25/25$ & $25/25$ & $0/25$ & $24/25$ & $0/25$ \\
popularity & 0.0 & 0.0 & shared & 0.00 & \emph{base rate} & $25/25$ & $25/25$ & $0/25$ & $25/25$ & $25/25$ \\
simpson\_cat & 16.5 & 22.8 & variant & 0.20 & \emph{Sparse} & $50/50$ & $50/50$ & $0/50$ & $50/50$ & $50/50$ \\
random restriction & $-0.3$ & $-0.6$ & variant & 0.00 & \emph{Unsupported} & $50/50$ & $50/50$ & $0/50$ & $50/50$ & $49/50$ \\
confound\_seg & 59.5 & 0.0 & variant & 0.85 & \emph{Unmatched} & $13/13$ & $12/13$ & $0/13$ & $13/13$ & $12/13$ \\
continuous, 3/3.5/4\% & 30.9 & $-0.6$ & variant & 0.75 & \emph{Unmatched} & $36/37$ & $34/37$ & $0/37$ & $36/37$ & $34/37$ \\
\midrule
\multicolumn{6}{@{}r}{\emph{over all $300$ samples}} & $\mathbf{299}$ & $283$ & $100$ & $\mathbf{298}$ & $257$ \\
\multicolumn{6}{@{}r}{\emph{resampling models written}} & $\mathbf{0}$ & $5$ & $0$ & & \\
\bottomrule
\end{tabular}}
\end{table}

\begin{table}[!htbp]
\centering
\begin{minipage}[t]{0.56\textwidth}
\centering
\caption{Certify decisions on the intervention tier, $108$ samples, as target matches\,/\,of; the last two columns
count the samples whose verdict matches the planted type. Ours names seventeen of the base rates \emph{Unsupported},
where the call is undecided.}
\label{tab:interv}
{\scriptsize\setlength{\tabcolsep}{1.8pt}
\begin{tabular}{@{}lccccc@{}}
\toprule
& \multicolumn{3}{c}{matches target\,/\,of} & \multicolumn{2}{c}{verdict matches} \\
\cmidrule(lr){2-4}\cmidrule(l){5-6}
type & ours & composite & held-out & ours & composite \\
\midrule
Matched & $18/18$ & $18/18$ & $18/18$ & $18/18$ & $18/18$ \\
Shared & $18/18$ & $18/18$ & $18/18$ & $18/18$ & $16/18$ \\
base rate & $18/18$ & $18/18$ & $0/18$ & $1/18$ & $0/18$ \\
Sparse & $18/18$ & $18/18$ & $0/18$ & $18/18$ & $18/18$ \\
Unsupported & $18/18$ & $18/18$ & $14/18$ & $18/18$ & $18/18$ \\
Unmatched & $18/18$ & $18/18$ & $0/18$ & $18/18$ & $18/18$ \\
\midrule
\textbf{all} & $\mathbf{108/108}$ & $\mathbf{108/108}$ & $50/108$ & $\mathbf{91/108}$ & $88/108$ \\
\textbf{types matched} & $\mathbf{6/6}$ & $\mathbf{6/6}$ & $2/6$ & & \\
\textbf{models written} & $\mathbf{0}$ & $1$ & $0$ & & \\
\bottomrule
\end{tabular}}
\end{minipage}\hfill
\begin{minipage}[t]{0.42\textwidth}
\centering
\caption{Certify decisions over the $408$ generator and intervention samples: ours, three ablations (middle) and the
procedures (bottom).}
\label{tab:severity}
{\scriptsize\setlength{\tabcolsep}{2.6pt}
\begin{tabular}{@{}lrrr@{}}
\toprule
 & match & f.\,decline & f.\,cert. \\
\midrule
$N{\times}N$ stats, $N{\times}N$ $M$ & $\mathbf{407}$ & $0$ & $1$ \\
\midrule
$N{\times}N$ stats, perm $M$ & $404$ & $0$ & $4$ \\
perm with $N{\times}N$ stats over $T$ & $340$ & $0$ & $68$ \\
perm alone & $263$ & $67$ & $78$ \\
\midrule
composite & $391$ & $13$ & $4$ \\
held-out & $150$ & $0$ & $258$ \\
\bottomrule
\end{tabular}}
\end{minipage}
\end{table}

$N\times N$ matches the target on $407$ of the $408$ planted samples (\cref{tab:exp,tab:interv}); the exception is a
generator window confound whose uniqueness product ends just past the bar, at $21.5$, a type-I error of the kind
\cref{thm:iut} bounds by $\alpha$. The composite matches $391$, with four false certificates: that confound and three
others, on which its unique-gain CRT rejects at $p=0.02$ to $0.045$ (\cref{app:baselinelimits}). It also declines the thirteen seasonal \emph{Shared} samples, on which its finite-sample
call does not return shared. Held-out validation, which asks only whether a hypothesis predicts,
certifies every \emph{base rate}, \emph{Sparse} and \emph{Unmatched} sample and matches $150$.
\cref{tab:severity} shows what the components buy: the permutation test alone, our $M$ with none of our statistics,
certifies every \emph{Sparse} sample and rejects nearly every \emph{Shared} one; with our other three statistics added
it still certifies four confounds, because shuffling the pairing across the matrix destroys the within-group structure
$\Euniq$ measures, and run over all of $T$ instead of the scope it issues $68$ false certificates.

The language-model tier tests hypotheses we did not write (\cref{tab:indep}). $N\times N$ makes all $285$
decisions the experts' labels call for, certifying the $70$ real associations with no false certificate, against
$268$ for the composite, also with none, and $74$ for held-out validation, which certifies $279$ of the $285$,
$210$ of them falsely.

\begin{table}[htbp]
\caption{The $285$ labeled hypotheses language models wrote over Amazon Electronics ($|U|=400$, $|V|=2000$;
\cref{app:indep}), scored against the human experts' labels as in \cref{tab:exp}; the three statistics are medians over each type, the e-values on a log
scale. Held-out validation compares each unit's test with the same test on as many random catalogue items
(\cref{app:indep}).}
\label{tab:indep}
\centering
{\scriptsize\setlength{\tabcolsep}{3.4pt}
\begin{tabular}{@{}lrrrrrrrrr@{}}
\toprule
& & & & & \multicolumn{3}{c}{decision matches target} & \multicolumn{2}{c}{verdict matches} \\
\cmidrule(lr){6-8}\cmidrule(l){9-10}
type & $\log_{10}\Esig$ & $\log_{10}\Euniq$ & $\hat c$ & $n$ & ours & composite & held-out & ours & composite \\
\midrule
\emph{Matched} & $34.2$ & $15.4$ & $0.552$ & $49$ & $\mathbf{49}$ & $\mathbf{49}$ & $\mathbf{49}$ & $47$ & $49$ \\
\emph{Shared} & $48.3$ & $0.0$ & $0.564$ & $21$ & $\mathbf{21}$ & $4$ & $20$ & $21$ & $1$ \\
\emph{base rate} & $0.6$ & $0.0$ & $0.000$ & $21$ & $\mathbf{21}$ & $\mathbf{21}$ & $0$ & $21$ & $0$ \\
\emph{Sparse} & $12.8$ & $3.0$ & $0.154$ & $118$ & $\mathbf{118}$ & $\mathbf{118}$ & $5$ & $118$ & $66$ \\
\emph{Unsupported} & $0.5$ & $0.0$ & $0.000$ & $23$ & $\mathbf{23}$ & $\mathbf{23}$ & $0$ & $22$ & $6$ \\
\emph{Unmatched} & $3.4$ & $0.3$ & $0.027$ & $43$ & $\mathbf{43}$ & $\mathbf{43}$ & $0$ & $39$ & $18$ \\
\emph{Subthreshold-Matched} & $0.9$ & $2.0$ & $0.000$ & $10$ & $\mathbf{10}$ & $\mathbf{10}$ & $0$ & $10$ & $1$ \\
\midrule
total &  &  &  & $285$ & $\mathbf{285}$ & $268$ & $74$ & $\mathbf{278}$ & $141$ \\
false certificates &  &  &  &  & $\mathbf{0}$ & $\mathbf{0}$ & $210$ &  &  \\
\multicolumn{5}{@{}r}{\emph{resampling models written}} & $\mathbf{0}$ & $285$ & $0$ & & \\
\bottomrule
\end{tabular}}
\end{table}

The certificates are sound where they are claimed. On nulls made true by construction---each unit's outcome
relabeled uniformly within its pool, each pair's outcomes swapped with probability one half---the two e-values reach
$1/\alpha$ on at most $0.29\%$ of the draws on any tier, against $\alpha=5\%$, and $\Esig$ on at most $0.14\%$ of
draws from a no-lift null without exchangeability, where a ratio term reaches $97\%$ (\cref{app:nullcheck}). The invariance
call's level is checked by simulation at the ceiling (\cref{app:einvderiv}),
and on the planted cases it calls shared $119$ of the $136$ scored samples whose units share a restriction, the
other seventeen undecided, and none of the $272$ whose units do not. These bounds are per-hypothesis; for a \emph{batch}, e-BH at the same level keeps every
language-model certificate's lift, and the two-part claim of all but one \emph{Matched} certificate (\cref{app:ebh}).

Our verdicts add a diagnosis: they name the audited type on $298$ of the $300$ generator samples, $91$ of the $108$
intervention samples and $278$ of the $285$ language-model ones, against $257$, $88$ and $141$ for the composite's (the
last columns of \cref{tab:exp,tab:interv,tab:indep}).

\FloatBarrier   
\section{Conclusion}\label{sec:conclusion}
$N\times N$ e-valuation extends a core idea it shares with the permutation test
\citep{fisher1935,pitman1937} and \citet{xu2024learning}---using the units under one hypothesis as nulls for each other---with four
statistics: $\Esig$, $\Euniq$, the invariance call and the certified share. Each is read by a test that
certifies, and together they route every hypothesis to a diagnostic verdict. The exact permutation test shares the
same cross-prediction matrix, so the four statistics are our only addition to it; patched with them it still issues
four false certificates (\cref{tab:severity}). No resampling model is written per hypothesis at any point, by
construction: \cref{alg:certify} takes the hypothesis, the data and four constants---$\alpha$, two thresholds and a
noise ceiling---and has nowhere to put one. That costs no correct decisions: over the three tiers $N\times N$
makes $692$ correct decisions in $693$, against $659$ for a composite of CRT tests that needs a resampling model per
hypothesis and $224$ for held-out validation. Two things bound what that buys. Restrictions that nearly coincide read as
shared, because on the observed outcomes they nearly are; four hypotheses the experts label \emph{Unmatched} pass the
call, and coverage withholds one of them and the gain test three. And a certificate asserts a lift, not its size: $\Esig$ clears $1/\alpha$ on a small lift once the scope
is large enough (\cref{sec:decision}), and a declared minimum lift, which we leave to future work, would bound its size. The demonstration uses two
categories of one corpus, so whether the split survives outside set-valued outcomes we have not shown.

\subsection*{Reproducibility statement}
Every number here is produced by a script over a synthetic generator or over the public Amazon
Reviews 2023 corpus \citep{hou2024amazon}, so reproduction needs the generator, the corpus split, the
certifier and the procedures. The world model and all of its parameters---catalogue size, population,
categories, brands, popularity exponent, outcome length, follow rate and scope weights---are stated in
\cref{app:exp}, the construction sequence in the order the generator executes it is
\cref{app:repro:world}, and each planted rule is given in \cref{tab:hyps,tab:extrahyps}; the corpus
split and the intervention are in \cref{app:repro:world}, and the language-model tier's corpus, tasks, frozen sets and labels in \cref{app:indep}. The certifier is \cref{alg:certify} together with the four statistics of
\cref{sec:statistics}, and the tolerances are in \cref{app:tolerances}. Every table states its own
denominator. The composite's model-X tests are handed an oracle sampling law on the planted tiers;
\cref{app:crtgates} defines the composite and \cref{app:crtlimits} says what its covariate tests can and cannot express.

\subsection*{AI use statement}
In this work, we used generative AI tools to implement methods (the certifier, the synthetic
generator, and the reference procedures of \cref{sec:exp}); to generate the synthetic data sets
those scripts produce; to propose the hypotheses of the language-model tier and restate item-list hypotheses as tests on items, in sessions
given only the task, the data and, for the restatements, the item-list hypotheses (\cref{app:indep}); to help \emph{refine} the conceptual framework---surfacing the Simpson's-paradox
failure mode that $\Euniq$ alone cannot reject, and helping refine the coverage signal (\cref{eq:cover}) that resolves it, which the authors designed---and to refine hypotheses; to supply
ingredients for the proofs of our mathematical claims and assist in writing those proofs
(\cref{app:theory} and the appendix bounds); to give feedback on the experiments; and to
interpret results. We have not used generative AI tools to design or propose the algorithm, to
formulate the mathematical claims, to design the experiments, to assist with translation, to clean
or reformat a dataset, or to support qualitative or thematic data analysis. Additionally, we used
generative AI tools to write and edit software code, to draft parts of this paper, to edit it for
readability and suggest its structure, to identify and summarise related literature, and to suggest
experimental parameters.

We have reviewed all AI-assisted work. Every theorem statement and proof was re-derived and checked
by the authors. Every number reported in this paper is produced by a script in our experiment artifact rather
than by a model, and each was checked against that script's raw output; arithmetic stated in
prose---ratios, break-even points---was recomputed from those outputs rather than carried over. We take
responsibility for the final content of this work, including text, claims or artifacts produced
with the aid of generative AI.

\ificlrfinal\else
\subsection*{Author contributions}
Except all the authors in the author list, yzhong36@jhu.edu also contributed to this work as
the project intern. She implemented the initial version of the system and should be considered as an
author too. Unfortunately, we couldn't find her open review profile thus wasn't able to add it to the
author list before deadline. Thus acknoledge her contribution here
\fi

\bibliographystyle{iclr2027_conference}
\bibliography{refs}

\clearpage
\appendix

\section{Notation}\label{app:notation}
Symbols appear in the order the body introduces them, each listed once at the definition that
governs its later uses; \cref{tab:objects} holds the objects they are derived from.

\begin{table}[htbp]
\caption{Objects in the certification problem, with the running recommendation example. Both $\mathrm{apply}$ and the restriction it returns may be non-deterministic and noisy.}
\label{tab:objects}
\centering
\scriptsize
\begin{tabularx}{\linewidth}{@{}l>{\hsize=0.80\hsize\raggedright\arraybackslash}X>{\hsize=1.34\hsize\raggedright\arraybackslash}X>{\hsize=0.86\hsize\raggedright\arraybackslash}X@{}}
\toprule
Object & Notation & Meaning & Example \\
\midrule
Unit & $i$ & a member of the population & a user \\
Test set, scope & $T$; $U=\{i:\mathrm{in\_scope}(x_i){=}1\}\subseteq T$ & $T$ is the evaluation population & all users; those matching it \\
Input & $x_i$ & what is known about a unit & interaction history \\
Outcome & $y_i$ & the held-out target: a \emph{set} of $|y|$ elements & the next $|y|$ interactions \\
Hypothesis & $H=(\mathrm{in\_scope},\mathrm{apply})$ & a scope test plus $\mathrm{apply}$ & ``a phone-buyer buys a case'' \\
Apply & $r=\mathrm{apply}(x)$ & normally the \emph{expensive} step; returns a restriction & compile the pattern to a filter \\
Restriction & $r$, $r(y)\ge0$ & what $\mathrm{apply}$ returns; its \emph{cheap} score rates $y$ & how well $y$ fits it \\
Collection & $A_B=\{a_i\}_{i\in B}$ & $Y_U=\{y_i\}_{i\in U}$; similarly $R_U$ & outcomes over $U$ \\
\bottomrule
\end{tabularx}
\end{table}

\begin{table}[htbp]
\caption{Every derived symbol, with where it is defined.}
\label{tab:notation}
\centering\footnotesize
\begin{tabularx}{\linewidth}{@{}llX@{}}
\toprule
Symbol & Defined & Meaning \\
\midrule
\multicolumn{3}{@{}l}{\emph{Sets and sizes}}\\
$A_B$ & \cref{tab:objects} & collection convention: $A_B=\{a_i\}_{i\in B}$ for any symbol $a$ and index set $B$; hence $Y_U,Y_V$ (outcomes) and $R_U$ (restrictions) \\
$U,\;T$ & \cref{sec:method} & in-scope set and test set, $U\subseteq T$ \\
$V$ & \cref{sec:statistics} & reference: units drawn uniformly from $T\setminus U$, $|U|$ of them by default, all of them if fewer \\
$B_j,\;\ell$ & \cref{sec:statistics} & unit $j$'s own block of $\ell=\lfloor|V|/|U|\rfloor$ references, disjoint across units \\

$M[i][j]$ & \cref{eq:matrixdef} & cross-prediction matrix, $r_i(y_j)$; the diagonal is each unit's own restriction \\
\midrule
\multicolumn{3}{@{}l}{\emph{The variance model}}\\
$\sigma_{\mathrm r}^2$ & \cref{eq:invmodel} & variance of the restrictions' overall levels; zero when restrictions are shared \\
$\sigma_{\mathrm y}^2$ & \cref{eq:invmodel} & variance of the outcomes' overall ease; the signal a shared restriction produces \\
$\sigma_{\mathrm{ry}}^2$ & \cref{eq:invmodel} & restriction $\times$ outcome interaction: which restriction fits which outcome; zero when restrictions are shared \\
$S_{\mathrm r}$ & \cref{eq:sr} & restriction spread on a common outcome, $\sigma_{\mathrm r}^2+\sigma_{\mathrm{ry}}^2$; $0$ for one restriction \\
$\mu$ & \cref{eq:invmodel} & mean score off the diagonal; the noise ceiling is $k\mu$ \\
$\bar M$ & \cref{sec:statistics} & the mean off-diagonal cell of $M$, which estimates $\mu$ \\
\midrule
\multicolumn{3}{@{}l}{\emph{The significance baseline}}\\
$\bar r_j$ & \cref{eq:esig} & restriction $r_j$'s mean score over its block $B_j$ of out-of-scope outcomes \\
\midrule
\multicolumn{3}{@{}l}{\emph{The four statistics}}\\
$\ejsig,\;\eijuniq$ & \cref{eq:esig,eq:emech} & per-unit significance term ($1$ plus the own score less its block's mean) and per-pair uniqueness term (own pairing against the swapped one) \\
$\mathrm{bet},\;\lambda$ & \cref{eq:agg} & betting combination: the mean over stakes $\lambda\in\Lambda$ of $\prod_t(1+\lambda(e_t-1))$ \\
$\Esig$ & \cref{eq:agg} & $\mathrm{bet}$ of the $\ejsig$: lift over the base rate outside the scope; an e-value for no lift, for independent units \\
$\Euniq$ & \cref{eq:agg} & $\mathrm{bet}$ of the $\eijuniq$: restriction matched to its unit; an exact e-value, no covariate model \\
call & \cref{eq:call} & the invariance call: \emph{shared} when $\widehat D+z\,\widehat{\mathrm{se}}<0$, with $\widehat D=\widehat S_{\mathrm r}-(k\bar M)^2$, read at $\alpha$ \\
$s_j$ & \cref{eq:cover} & unit $j$'s sign against its comparator $v_j$, a distinct random member of $V$: $+1$ if its own outcome fits its restriction better, $-1$ if worse, $0$ on a tie \\
$\hat c$ & \cref{eq:cover} & the certified share: a lower confidence bound on the share of units with a gain; $\hat c\ge\tau_{\mathrm{cov}}$ is required to certify, never sufficient \\
\midrule
\multicolumn{3}{@{}l}{\emph{User inputs: the level, two thresholds and a noise ceiling}}\\
$\alpha,\;1/\alpha$ & \cref{sec:method} & level, and the threshold every e-value must clear to certify \\
$k$ & \cref{sec:statistics} & noise ceiling: restrictions whose scores on a common outcome spread by less than $k\mu$ count as one. Declared; it fixes what counts as one restriction, so report it with any \emph{Shared} certificate \\
$\gamma$ & \cref{tab:classes} & share of $T$ above which a scope is the population and reads as \emph{base rate}; $0.9$ by default \\
$\tau_{\mathrm{cov}}$ & \cref{tab:classes} & share of units the certified share must show carrying the gain, set by the user; if it is not shown the certificate is withheld (\emph{Sparse}) \\

\bottomrule
\end{tabularx}
\end{table}

\section{Experiment details}\label{app:exp}

\textbf{Generator.} A catalog of $1500$ items, each with a category (of $40$), a brand (of $60$), and a Zipf popularity ($s{=}1.05$). Each of $|T|{=}16{,}000$ users draws a scope from $\{$complement, brand, category, holiday, exposed, base$\}$ with weights $(0.12,0.03,0.03,0.03,0.03,0.76)$, so every rule scope holds at least $3\%$ of $T$ and the certifier draws $|U|=400$ of its units; a history $x_i$ unioning the user's recent item, favourite-brand and recent-category picks with popularity draws, so $|x_i|$ varies (\cref{app:repro:world}); and a held-out outcome set $y_i$ ($6$ items) whose mechanism is the scope's own with probability $p_{\mathrm{follow}}{=}0.9$ and popularity noise otherwise; \cref{app:repro:world} gives the per-item draw within a mechanism. The structural confound is planted separately: each ``segment'' covers a different ${\sim}60\%$ subset of a \emph{low-popularity} promoted set $P$---a structural confound must carry a scope-level lift, which requires its target set to be rare outside the scope---exposed users' $y_i\subset P$, and the segment$\to$outcome score factors as (segment overlap with $P$)$\times$(item appeal) with no pairing term, so no unit's restriction is matched to its own outcome, as the definition requires. Each shared restriction is returned with per-unit noise, every item replaced with a probability drawn per unit from $[0,0.02]$, and a further rule hands every $26$th user, whatever its scope, its own popularity-weighted random restriction (\emph{Unsupported} by design).

\textbf{Certifier.} $r_i(y)=$ recall of restriction $i$ on outcome set $y$
($|\text{retained}_i\cap y|/|y|$); $\Esig$, $\Euniq$, the invariance call and the certified share exactly as in \cref{sec:method},
against a reference of $|V|=1000$ out-of-scope units on both planted tiers, so $\ell=2$ at $|U|=400$ (the popularity
rule's scope is all of $T$, which \cref{alg:certify} returns as \emph{base rate} before drawing any); the blocks, the
pairs and the comparators are drawn from a seed fixed per sample before any outcome is read, and the matrix spans all of $U$. The generator predates the
certifier and was written to mimic a real corpus and nothing else: none of
\cref{app:einvderiv}'s conditions was considered in building it---with $|y|=6$ each cell takes one
of seven values, and restrictions of different breadth carry different cell variance.

\textbf{Tolerances.} $\alpha{=}0.05$, $\tau_{\mathrm{cov}}=0.5$, $\gamma=0.9$ and
$k=0.3$ are held fixed across every world and every rule. $\tau_{\mathrm{cov}}$ is a declaration rather than a measurement: it is the share of in-scope
units a user requires to carry the gain, and a certificate needs the certified share
to show more than that share. At $0.5$ we ask for a majority; raising it rejects real rules and never miscertifies
(\cref{tab:tolerances}).

\textbf{Procedures.} We compare two. \emph{Naive held-out validation} is a one-sided $z$-test that the rule's mean
held-out recall exceeds that of a random same-size set; it certifies anything predictive, all four fakes included. The
\emph{composite} (\cref{app:crtgates}) puts a scope randomization test in place of $\Esig$, a model-X CRT
\citep{candes2018}, given the generator's sampling law as an oracle, in place of $\Euniq$ and a finite-sample betting
bound in place of the invariance call, uses the certified share for coverage, as we do, and routes the four through
\cref{tab:classes}. The \emph{exact within-scope permutation test} is an ablation: it reads the
same $M$ we do, under the permutation of the pairing that \citet{mantel1967} tests two matrices with, so it differs from
our certifier in the statistics taken off the matrix and not in the null, and isolates what those statistics contribute.

\textbf{Scoring.} On the two planted tiers the label is the planted type as audited by experts, and every method's
certify decision is compared with the target: it matches when it certifies a \emph{Matched} or \emph{Shared} sample
or rejects any other. The decisions are whether our verdict or the composite's certifies, its CRT's law
per-family on the generator and a corpus-wide redraw on the intervention tier; held-out validation's one-sided $z$; and,
for the ablations, the permutation test's $p\le\alpha$. Verdict names, ours and the composite's, are counted
separately (\cref{app:verdicts}). On the language-model tier the label is the type human experts
assign (\cref{app:indep}).

\textbf{Seeds.} The planting routine draws while iterating a set of string identifiers, so the world
it builds depends on the interpreter's string-hash seed; we pin that seed for every run.

\subsection{World construction}\label{app:repro:world}
With $1500$ items, $16{,}000$ users, $40$ categories, $60$ brands, in this order:
\begin{enumerate}\itemsep1pt
\item Category and brand for each item, uniform and independent.
\item Popularity: draw a uniform permutation of the items and give the item at rank $\rho$ weight
  $\rho^{-1.05}$, normalised. \textbf{Rank is assigned to a shuffled order}, so popularity is
  independent of item index; an implementation that ranks by index correlates popularity with every
  other index-derived quantity.
\item \emph{gift\_items}: $6$ items sampled from category $0$.
\item \emph{promo}: $40$ items sampled from the \emph{lower-popularity half} of the non-gift items
  ordered ascending by popularity. A structural confound must carry a strong scope-level lift, which
  requires its target set to be rare outside the scope.
\item Per-item complements: $4$ items sampled uniformly from the other $1499$, per item.
\item Segments: for each of $10$ segments, a different $24$-item subset of \emph{promo}
  ($60\%$). All segments predict \emph{promo} about equally well but disagree on which of its
  items: restrictions that differ across units with none matched to its own outcome, as the
  definition of a structural confound requires.
\item \emph{top\_pop}: the $40$ highest-popularity items.
\item Per user, in order: scope label from $\{$complement, brand, category, holiday, exposed,
  base$\}$ with weights $(0.12,0.03,0.03,0.03,0.03,0.76)$; a recent item; a favourite brand; a
  recent category; a segment; a heavy-user flag at rate $0.25$. Then the mechanism label: the
  scope's own rule with probability $0.9$ if the scope is not \emph{base}, otherwise
  \emph{popularity}.
\item Input $x_i$: the recent item, $3$ items of the favourite brand, $3$ of the recent category,
  and $15$ popularity-weighted draws, unioned as a \emph{set}---so $|x_i|$ varies (19--22 at
  seed~0) rather than being fixed.
\end{enumerate}
Outcomes are drawn last, after every user record exists, because the mechanism pools are read from
the completed world. For a unit with a mechanism, draw $6$ items, each from the mechanism's pool
with probability $0.95$ and popularity-weighted otherwise; for a \emph{popularity} unit, draw $6$
popularity-weighted. Accumulate into a \emph{set} and top up with popularity draws until it holds
exactly $6$, so $|y_i|=6$ for every unit.

Weighted sampling without replacement is by rejection---draw one at a time, keep if new, give up
after $60k$ attempts for a target of $k$. With a Zipf tail this matters: sampling by an exact
without-replacement algorithm gives a measurably different popularity mix.

\paragraph{The Amazon corpus.} The intervention tier runs on the Video-Games category of Amazon Reviews 2023
\citep{hou2024amazon}, taken from the user-review and item-metadata dumps at parent-ASIN granularity; the
language-model tier runs on its Electronics category (\cref{app:indep}). We keep users with at least eight
interactions and at least eight distinct parent-ASINs after de-duplication, retaining one
occurrence per ASIN at its latest timestamp, and split them into a proposal half of $20{,}462$
users and a disjoint certification half of $|T|=20{,}323$. Outcomes are the last three purchases by
timestamp, so $|y|=3$; the de-duplication is what makes that exactly three, since a user may touch
one ASIN repeatedly. A unit's history and its outcome are disjoint in both item and time---the
outcome is the final three distinct ASINs, the history everything strictly before them---so a
restriction compiled from the history is never scored on an item it was compiled from. Popularity, per-period popularity and co-occurrence are computed
on the proposal half alone, and every claim is relative to that subpopulation rather than to Amazon
as a whole. The intervention on it plants the six types of \cref{tab:hyps}.

\begin{table}[htbp]
\caption{The planted hypotheses.}
\label{tab:hyps}
\scriptsize
\begin{tabularx}{\linewidth}{@{}>{\raggedright\arraybackslash}p{1.65cm}>{\raggedright\arraybackslash}p{1.95cm}X@{}}
\toprule
Type & Rule & Recommendation rule (and, for fakes, why it is fake) \\
\midrule
Matched\newline\emph{real} & complement & ``a user who just bought $X$ buys $X$'s complement next'': $\mathrm{apply}$ returns the per-item complement of the user's \emph{own} recent purchase, so $r_i$ fits unit $i$'s outcome and few others' \\
Matched\newline\emph{real} & brand & ``a user keeps buying from their favourite brand'': $\mathrm{apply}$ returns that user's favourite-brand items \\
Matched\newline\emph{real} & category & ``a user re-engages a recently browsed category'': $\mathrm{apply}$ returns that user's recent-category items \\
Shared\newline\emph{real} & seasonal & ``holiday shoppers buy the gift-wrap line'': $\mathrm{apply}$ returns \emph{one} gift set, near-identical across in-scope users---real but constant \\
base rate\newline\emph{fake} & popularity & ``everyone buys the best-sellers'': $\mathrm{apply}$ returns the globally most-popular items (same for all). \emph{Why fake:} it merely restates the base rate---in-scope users buy them no more than the population does \\
base rate\newline\emph{fake} & heavy\_pop & ``the best-sellers, claimed of a sub-population'': the same globally popular items, claimed only of a random quarter of the users. \emph{Why fake:} popular items carry no lift over the base rate even there---a base rate dressed up as a targeted rule \\
Unmatched\newline\emph{fake} & confound\_\allowbreak segment & ``each user buys items matching their market segment'': $\mathrm{apply}$ returns a per-user segment set and \emph{looks} personalized. \emph{Why fake:} purchases are driven by a shared, exogenous promoted campaign, not the segment; every segment overlaps it, so $r_i$ predicts others as well as its own \\
Sparse\newline\emph{fake} & simpson\_\allowbreak cat & ``users buy the complement of their recent item,'' but the mechanism holds only when the recent item's category lies in a fixed $30\%$ of categories. \emph{Why fake:} as stated it claims the whole scope, and it is predictive in aggregate only because that subgroup carries it. Because the subgroup is a \emph{covariate} the hypothesis is \emph{refinable} \\
Matched\newline\emph{real} & complement\_\allowbreak cat & \emph{simpson\_cat} after adding the condition the planted subgroup supplies directly. A genuine per-unit rule on the narrowed scope \\
Unsupported\newline\emph{fake} & random restriction & every $26$th user, whatever its scope, is handed its own popularity-weighted random restriction. \emph{Why fake:} the restrictions are unrelated to the outcomes, so nothing in them predicts \\
\midrule
\multicolumn{3}{@{}l}{\emph{Planted by intervention on the Amazon corpus}}\\
Matched & co-purchase kNN & each unit's restriction is drawn from its own item-KNN neighbourhood, so every restriction is matched to the unit holding it \\
Matched & latent set & a latent $10$-item set per unit, of which the restriction exposes $4$ while the outcome is drawn from all $10$---the diagonal is diluted before $p_{\mathrm{follow}}$ reaches it \\
Shared & one shared set & one restriction, near-identical across the whole scope \\
base rate & popularity head & the popularity head, with nothing planted, so in- and out-of-scope users hit it at the same rate \\
Unmatched & group subset & units split into eight latent groups; $\mathrm{apply}$ returns the unit's group subset of a promoted set $P$ while outcomes are drawn from $P$ whole, so every group predicts about equally and none is matched to its own unit \\
Sparse & $30\%$ subgroup & the mechanism reaches only $30\%$ of the scope; the rest keep their real outcomes \\
Unsupported & random per unit & outcomes left real; each in-scope unit is handed its own popularity-weighted random restriction \\
\bottomrule
\end{tabularx}
\end{table}

\begin{table}[htbp]
\caption{The added-threat instances \cref{tab:exp} scores, in the register of \cref{tab:hyps}. Three
generators written from scratch, sharing no popularity model with the main one: out-of-scope outcomes
are uniform, with no categories, brands, histories or mechanism pool. Each scope holds at least $3\%$
of $T$ and the certifier reads $400$ of its units; together they give $84$ of the $300$ generator
samples.}
\label{tab:extrahyps}
\scriptsize
\begin{tabularx}{\linewidth}{@{}>{\raggedright\arraybackslash}p{1.65cm}>{\raggedright\arraybackslash}p{1.95cm}X@{}}
\toprule
Type & Rule & Recommendation rule (and, for fakes, why it is fake) \\
\midrule
Matched\newline\emph{real} & latent readout & ``a user buys within the interest group they belong to'': each unit joins one of $40$ latent groups, each owning a $10$-item set from which that unit's outcome is drawn. $\mathrm{apply}$ returns the $10$-item set of the group a \emph{noisy readout} names, correct with probability $0.9$. The group drives both the observed covariate and the outcome, and nothing else does. Not fake: the association is real and matched to its unit, and the readout noise is what keeps it from collapsing into \emph{complement}. Scope $3\%$ \\
\midrule
Unmatched\newline\emph{fake} & continuous window & ``each user buys the items in their own slice of a quality ranking'': one exogenous score ranks the catalogue, in-scope outcomes are drawn from its top $40$, and $\mathrm{apply}$ returns a per-unit $10$-item window of that ranking at a uniformly drawn offset. \emph{Why fake:} the shared ranking drives the purchases, not the unit's window; every window overlaps the same head, so $r_i$ predicts other units as well as its own. The shared factor is continuous where \emph{confound\_segment}'s is a discrete segment. Scopes $3\%$, $3.5\%$ and $4\%$ \\
\midrule
Shared\newline\emph{real} & constant effect & ``a promotion lands on an audience'': one $30$-item set, near-identical for every in-scope unit and followed with probability $0.9$. Real and constant: its lift is measured against units outside the scope, who do not follow it. Scopes $3\%$, $50\%$ and $80\%$ of $T$, all under the $\gamma=0.9$ line where a scope becomes the population \\
\bottomrule
\end{tabularx}
\end{table}

\subsection{What each signal buys}\label{app:ablation}\label{app:crtlimits}

\begin{table}[htbp]
\caption{Verdicts of \cref{tab:classes} that become indistinguishable when one statistic is
removed. All four are load-bearing: $\dagger$ marks a merge that joins a \emph{certify} verdict
with a \emph{reject} one, and every removal makes one.}
\label{tab:necessity}
\centering
\small
\begin{tabular}{@{}llc@{}}
\toprule
removed & verdicts it merges & lost \\
\midrule
$\Euniq$ & Base rate = Unsupported = Subthreshold-Matched & $\mathbf{5}$ \\
 & Shared = Unmatched = Matched = Sparse$^\dagger$ & \\
$\Esig$ & Base rate = Shared = Sparse$^\dagger$ & $\mathbf{5}$ \\
 & Subthreshold-Matched = Matched = Sparse$^\dagger$ & \\
 & Unsupported = Unmatched & \\
call & Base rate = Unsupported & $3$ \\
 & Shared = Unmatched = Sparse$^\dagger$ & \\
$\hat c$ & Matched = Sparse$^\dagger$ & $2$ \\
 & Shared = Sparse$^\dagger$ & \\
\bottomrule
\end{tabular}
\end{table}

$\Euniq$ separates a per-unit mechanism from a unit-invariant one but is blind inside the latter:
a shared restriction makes every row of $M$ equal, forcing $\Euniq=1$ (\cref{app:unitinv}). The
invariance call reads that degeneracy as the signature of sharing rather than as a failure, which is what admits
\emph{Shared} and \emph{base rate} at all---and why dropping it merges \emph{Shared} into
\emph{Unmatched}.
The certified share is what separates \emph{Sparse}: \emph{simpson\_cat}'s carrying minority pushes
$\Euniq$ past $1/\alpha$, so the case never reaches the branch the call gates.

On the planted tiers the call says shared for every \emph{Shared} sample and every generator
\emph{base rate}, but for only one of the eighteen intervention base rates, whose restrictions spread
by $0.24$--$0.31\bar M$ against the ceiling of $0.3\bar M$; the other seventeen are undecided and named
\emph{Unsupported}, a right decision under another name. On the language-model
tier the call is shared for all $21$ \emph{Shared} hypotheses.

\begin{table}[htbp]
\caption{The certify decision by planted type---the per-type split of \cref{tab:exp}'s certify column, with two
ablations. Our one miss is the window confound of \cref{sec:exp:results}; the composite declines thirteen \emph{Shared}
samples, and it and the permutation test given our statistics each certify four of the fifty \emph{Unmatched} ones.}
\label{tab:cover}
\centering\small
{\setlength{\tabcolsep}{4.5pt}
\begin{tabular}{@{}lccccc@{}}
\toprule
 & & & & \multicolumn{2}{c}{ablations} \\
\cmidrule(l){5-6}
planted type & ours & composite & held-out & perm & \shortstack{perm\\$+\,\hat c+{}$call} \\
\midrule
Matched & $50/50$ & $50/50$ & $50/50$ & $50/50$ & $50/50$ \\
Shared & $50/50$ & $37/50$ & $50/50$ & $\mathbf{1/50}$ & $50/50$ \\
base rate & $50/50$ & $50/50$ & $\mathbf{0/50}$ & $49/50$ & $50/50$ \\
Sparse & $50/50$ & $50/50$ & $\mathbf{0/50}$ & $\mathbf{0/50}$ & $50/50$ \\
Unsupported & $50/50$ & $50/50$ & $\mathbf{0/50}$ & $46/50$ & $50/50$ \\
Unmatched & $49/50$ & $46/50$ & $\mathbf{0/50}$ & $46/50$ & $46/50$ \\
\midrule
\textbf{types matched} & $\mathbf{5/6}$ & $4/6$ & $2/6$ & $1/6$ & $\mathbf{5/6}$ \\
\bottomrule
\end{tabular}}
\end{table}

\paragraph{What the composite misses.} Its unique-gain CRT reads a hypothesis only through $\mathrm{apply}$'s dependence
on the covariates it redraws, so on a unit-invariant hypothesis resampling changes nothing, and every generator
\emph{Shared} and \emph{base rate} sample gets $p=1$. That is the right answer to whether the gain is unit-specific
but no evidence that the restriction is shared: the random-restriction plant, whose restrictions differ from unit to
unit, gets $p=1$ as well, because the law holds the identifier they are keyed on. Sharedness is left to the
finite-sample call, which shows none of the thirteen seasonal restrictions shared. The CRT's law also has to be
written for each hypothesis: the composite needed five on the generator, one on the intervention tier and $285$ on the
language-model tier, and a stream of machine-proposed hypotheses would demand the same judgement for each.

\subsection{Limitations of dependence tests}\label{app:baselinelimits}
A model-X CRT and the exact permutation test ask whether unit-level dependence exists: the CRT whether the outcome
depends on the covariate that carries the mechanism, given the scope's covariates, and the permutation test whether
each unit's own restriction scores its outcome better than a randomly paired one. Both answer that question validly.
Neither answer is a scope-level certificate, and read as one it misleads in two opposite ways.

\textbf{A minority-carried association.} It has unit-level dependence, so both tests reject their nulls: the
permutation test certifies all fifty \emph{Sparse} samples of \cref{tab:cover}. The rejection is right for their nulls
and silent on whether the association holds across the claimed scope. The composite and
$N\times N$ both add the certified share, which withholds every planted \emph{Sparse} sample (\cref{app:choosecov}).

\textbf{A restriction shared by every unit.} It gives both tests nothing to move: the CRT's resampling cannot change a
restriction that is the same for everyone (\cref{app:crtlimits}), and permuting the pairing changes little when every
unit holds a score-equivalent restriction. The permutation test certifies one of the fifty \emph{Shared} samples, and
the CRT returns $p=1$ on all fifty. What certifies a shared restriction is a scope-level comparison with a sharedness
call: $N\times N$'s $\Esig$ and invariance call, or the composite's randomization gain test and finite-sample call,
which shows $37$ of the fifty shared.

\textbf{The burden.} Both distinctions need more than a dependence test: a per-unit coverage for the minority, and a
scope-level comparison with a sharedness call for the shared restriction. None of these needs a per-hypothesis model,
and the composite takes the certified share, a scope randomization test and a finite-sample call for them; only its
unique-gain CRT needs a resampling law written per hypothesis (\cref{app:crtlimits}), which does not scale to
machine-proposed hypotheses. $N\times N$ draws every distinction from the scope's matrix and its out-of-scope
reference, with no per-hypothesis design.

\subsection{The composite procedure}\label{app:crtgates}
The composite puts an existing test in place of three of our statistics and uses the fourth, the certified share, as
we do. The \emph{scope randomization test} replaces $\Esig$: in each of $200$ rounds every unit's own outcome is
replaced by one drawn uniformly from the out-of-scope reference, and the diagonal mean is compared with the rounds'.
The \emph{unique-gain CRT} replaces $\Euniq$: the model-X mechanism test, which redraws the covariates carrying the
mechanism from a law written for the hypothesis. The two enter the router of \cref{tab:classes} as $1/\alpha$ when
their $p\le\alpha$ and as $1$ otherwise. The invariance call is replaced by a finite-sample one: betting confidence
bounds on $S_{\mathrm r}$ and $\mu$ over disjoint blocks of six units, \emph{shared} when the upper bound on
$S_{\mathrm r}$ falls below $k^2$ times the squared lower bound on $\mu$, with $\alpha$ split between the two. The
certified share compares each unit's own outcome with an out-of-scope one---the scope randomization test's comparison,
made per unit---and needs no model. The composite so returns our verdicts, and only its unique-gain CRT needs a
resampling law. Every row below uses the certified share:

\begin{center}\small
{\setlength{\tabcolsep}{4pt}
\begin{tabular}{@{}lccc@{}}
\toprule
decisions matching the target / false certificates & generator & intervention & language model \\
\midrule
betting e-values, our call (ours) & $299/1$ & $108/0$ & $\mathbf{285}/0$ \\
the two tests, our call & $296/4$ & $108/0$ & $\mathbf{285}/0$ \\
the two tests, finite-sample call (composite) & $283/4$ & $108/0$ & $268/0$ \\
\bottomrule
\end{tabular}}
\end{center}

Putting the two tests in place of the e-values costs three false certificates, all on the generator: the unique-gain
CRT asks whether the outcome depends on the covariate carrying the mechanism, not whether each unit's own restriction is
the one that fits it, and passes four structural confounds as \emph{Matched}, at $p=0.02$ to $0.045$, where $\Euniq$
passes one of them. Putting the finite-sample call in place of the jackknife call costs certificates only: it is shared
on $37$ of the $50$ generator \emph{Shared} samples, all $18$ intervention ones and one of the $21$ language-model ones,
where the jackknife call is shared on all of them, so the composite declines the thirteen seasonal and seventeen
language-model \emph{Shared} samples; three more language-model ones reach a certificate through its unique-gain CRT,
named \emph{Matched}.

\subsection{What each verdict tells the user}\label{app:verdicts}
A certificate says that a hypothesis's lift is real; the verdict adds where that lift appears to come
from, which is what a proposer needs to decide what to try next. \Cref{tab:verdictuse} gives one
reading of each verdict and a step it suggests. The readings are examples rather than the only way
to interpret a verdict, and none is covered by \cref{thm:iut}: the guarantee concerns certification,
and the name is feedback.

A proposer that revises its hypotheses on these verdicts needs the data in three parts. It proposes on a
training part, whose outcomes it sees. Its feedback---each verdict with its details, such as a rejection for low
coverage despite high significance, which narrowing the scope may fix---comes from rerunning the certifier on a
validation part. The certificate is issued on a test part the proposer never sees, so a revised hypothesis is still
chosen without reading a certification outcome, as the guarantee requires (\cref{app:theory}). Our experiments give no
feedback, so they need no validation part; on the language-model tier the proposal half is the training part and the
certification half the test part (\cref{app:indep}).

\begin{table}[htbp]
\caption{One way to read each verdict, and a next step it suggests. Examples, not definitions.}
\label{tab:verdictuse}
\centering\scriptsize
\begin{tabularx}{\linewidth}{@{}l>{\raggedright\arraybackslash}X>{\raggedright\arraybackslash}p{3.6cm}@{}}
\toprule
Verdict & One reading & A step it suggests \\
\midrule
\emph{Matched} & The predictive factor comes from each unit: the hypothesis extracts a property of the unit that predicts that unit's own outcome. & Keep it; the per-unit restriction does the work. \\
\emph{Shared} & The predictive factor lies in the hypothesis itself, or in a property the units in its scope have in common, rather than in anything specific to one unit. & Keep it as a claim about the scope, not about each unit. \\
\emph{base rate} & The prediction comes from the population the data are drawn from, not from anything the hypothesis says. & Drop it; it restates what holds everywhere. \\
\emph{Sparse} & The hypothesis may hold for only part of its scope, or the scope is too small to show that most units follow it. & Narrow the scope to where it holds, or re-test on more units. \\
\emph{Unmatched} & The scope carries a real gain, but the factor behind it is not yet revealed: the restrictions differ and none is matched to its own unit. & Look for the common factor behind the gain. \\
\emph{Unsupported} & Neither one shared claim nor any gain: nothing the hypothesis names predicts. & Drop it or rethink the restriction. \\
\emph{Subthreshold-Matched} & Each unit's restriction fits its own outcome best, yet in scope no better than outside: a per-unit signal the scope does not sharpen, or one too weak to show at this size. & Widen or drop the scope, or test on more units. \\
\bottomrule
\end{tabularx}
\end{table}

\section{Guarantees: standing conditions and proofs}\label{app:theory}
\cref{prop:sig,prop:mech,prop:agg} and \cref{thm:iut} are stated in \cref{sec:decision}; this section gives the conditions they are proved under and the proofs.
The propositions and the e-value bounds of \cref{thm:iut} hold in finite samples. The uniqueness bound
holds under arbitrary dependence between units; the significance bound and the coverage claim need the i.i.d.\
assumption below, and so does the invariance call, whose bounds are asymptotic and which also needs a declared noise
ceiling; it is derived in \cref{app:einvderiv}.

\textbf{The i.i.d.\ assumption.} Conditional on the proposal data, the hypothesis and its scope, the
in-scope pairs $Z_i=(r_i,y_i)$ are independent and identically distributed, and independent of the
reference units $V$, which are themselves i.i.d.\ from the out-of-scope population. Three steps use it: the significance e-value, whose product needs its terms
independent (\cref{prop:agg}); the invariance call, whose two U-statistics need independent units
(\cref{app:einvderiv}); and the certified share, whose product needs each unit's sign against its comparator to be independent
of the others' (\cref{app:choosecov}). The uniqueness e-value does not: its pairs are exchanged by the null itself,
so its terms are independent given the outcomes whatever the dependence between units. Time and group membership are
properties of a unit, so drift and clustering need not violate the assumption: a hypothesis can put a period, a region
or a segment in its scope, and units drawn uniformly from the population a scope selects are i.i.d.\ however its
members differ over time or by group, the claim being about that population. What violates it is units not drawn
independently, such as several taken together from one user. The assumption concerns the data a certificate is issued
on, not the traffic a certified hypothesis is later applied to, and data from any traffic can be prepared to satisfy it,
as ours are: each user's purchases are grouped into one unit, and the generator draws its users independently given its catalogue.
Throughout, each restriction $r(\cdot)$ is a score in $[0,1]$ returned by $\mathrm{apply}$ from the
input and, when $\mathrm{apply}$ is stochastic, from proposal randomness independent of the outcomes.
The significance proof uses only means. The uniqueness proof conditions on $R_U$, which \emph{fixes} each restriction
and so a pair's two sums, and on each pair's outcomes as an unordered pair, which leaves which of them is $y_i$
random; what drives $\eijuniq$ to $1$ is its null, not the form of $r$. When the hypothesis is real, $r_j$ and $y_j$
are dependent---both are downstream of $x_j$---and what $\Euniq$ tests is whether that dependence is matched to the unit.

\textbf{Why a difference.} A ratio of a unit's score to a pool mean that includes it is an exact e-value when the
unit's outcome is exchangeable with its block's, but not under no lift. Its denominator is random: if an in-scope score
always equals the out-of-scope mean $c$ while out-of-scope scores vary, the ratio's mean is $\Ex[(\ell+1)c/(c+S)]>1$
by Jensen's inequality, $S$ being the block's sum. On nulls of that kind the ratio reaches $1/\alpha$ on up to $97\%$
of the draws while $\Esig$ stays at or below $0.14\%$ (\cref{app:nullcheck}). A bet on the difference needs nothing beyond
the means, so how in-scope and out-of-scope outcomes are distributed never has to be checked.

\begin{remark}[Identities under a unit-invariant hypothesis]\label{app:unitinv}
For a unit-invariant hypothesis, $r_i\equiv r$, so $M[i][j]=r(y_j)$ does not depend on $i$: every
row of $M$ equals $r(y_\cdot)$ and every column is constant. Hence, for every pair,
\begin{equation}
\begin{aligned}
&\eijuniq=\frac{2\,(r(y_i)+r(y_j))}{2\,(r(y_i)+r(y_j))}=1,\qquad
\Euniq=1,\\
&\widehat S_{\mathrm r}=0,\qquad \widehat D=-(k\bar M)^2.
\end{aligned}
\end{equation}
The last two follow with no approximation: every column is constant, so each column's variance
vanishes, in $\widehat S_{\mathrm r}$ and in every jackknife replicate, leaving $\widehat D$ and its
replicates at $-k^2$ times $\bar M^2$ and its replicates. The call is then \emph{shared} exactly when
$\bar M^2$ exceeds $z$ times its own jackknife error; $k$ cancels, so on a unit-invariant hypothesis
the call never turns on its value. The premise $r_a=r_b$ holds
exactly for a deterministic $\mathrm{apply}$ and only up to sampling agreement otherwise.
\end{remark}

\begin{proof}[Proof of \cref{prop:agg}]
Fix $\lambda\in\Lambda$. Each factor $1+\lambda(e_t-1)\ge1-\lambda>0$, since $e_t\ge0$ and $\lambda<1$. If the terms are
independent with $\Ex[e_t]\le1$, then $\Ex\big[\prod_t(1+\lambda(e_t-1))\big]=\prod_t\big(1+\lambda(\Ex[e_t]-1)\big)\le1$,
and a mean of e-values is an e-value, so $\mathrm{bet}$ is one. For $\Esig$: given the proposal data and the scope,
the triples $(r_j,y_j,Y_{B_j})$ are independent across $j$, because the units are i.i.d.\ and the blocks disjoint, and
each $\ejsig$ has mean at most $1$ under $H_0^{\mathrm{sig}}$ (\cref{prop:sig}). For $\Euniq$ no independence between units is
needed. Condition on $R_U$ and on the unordered pair of outcomes each pair of units holds. Under
$H_0^{\mathrm{uniq}}$ the pairing is exchangeable, and a uniform pairing is unchanged by swapping the outcomes within
any pair, so given this conditioning the orientations of the pairs are independent fair coins. Each $\eijuniq$ is a
function of its own pair's coin with conditional mean $1$ (\cref{prop:mech}), so every product has conditional mean
exactly $1$, and so does $\Euniq$.
\end{proof}

\begin{proof}[Proof of \cref{thm:iut}]
Both certificates require $\Esig\ge 1/\alpha$. Under $H_0^{\mathrm{sig}}$, for independent units, $\Esig$ is a valid e-value (\cref{prop:sig,prop:agg}), so Markov's inequality gives $\Prob(\Esig\ge 1/\alpha)\le\alpha$, and since every certificate lies in that event, $\Prob(\text{certify})\le\alpha$. A \emph{Matched} certificate also requires $\Euniq\ge 1/\alpha$, and under $H_0^{\mathrm{uniq}}$ the same argument via $\Euniq$ bounds it by $\alpha$, so it has probability $\le\alpha$ under either component of $H_0^{\mathrm{sig}}\cup H_0^{\mathrm{uniq}}$: the
\emph{intersection--union} principle \citep{berger1982iut}. The remaining conditions---$\hat c\ge\tau_{\mathrm{cov}}$, a \emph{shared} call and a scope within $\gamma|T|$---only shrink
the sets these bounds are taken over: a one-sided gate that
can only remove inherits the guarantee whatever value it is set to, and needs no calibration of its
own \citep[Fact 1.3(vi)]{ramdas2023}. The coverage claim is bounded by its own test: under
$H_0^{\mathrm{cov}}$ every certificate lies in $\{\hat c\ge\tau_{\mathrm{cov}}\}$, of probability at most
$\alpha$ (\cref{app:choosecov}). The attribution a \emph{Shared} certificate adds is bounded by the call, asymptotically (\cref{app:einvderiv}).
\end{proof}

\begin{proof}[Proof of \cref{prop:sig}]
Fix a unit $j\in U$. Its block $B_j$ is drawn without reading an outcome and independently of $j$, so each $y_k$,
$k\in B_j$, is an out-of-scope outcome independent of $(r_j,y_j)$, and $\Ex[r_j(y_k)]=\Ex[r_j(y_v)]$. Hence
$\Ex[\bar r_j]=\Ex[r_j(y_v)]$ and $\Ex[\ejsig]=1+\Ex[r_j(y_j)]-\Ex[r_j(y_v)]$, which is at most $1$ under
$H_0^{\mathrm{sig}}$; $\ejsig\ge0$ because scores lie in $[0,1]$. Nothing else about the outcomes' distribution is used.
\end{proof}

\begin{proof}[Proof of \cref{prop:mech}]
Fix a pair $(i,j)$ and condition on $R_U$ and on the unordered pair $\{y_i,y_j\}$ of outcomes it holds. The two sums
$s=M[i][i]+M[j][j]$ and $s'=M[i][j]+M[j][i]$ are then fixed up to which is which, since swapping the pair's outcomes
exchanges them. Under $H_0^{\mathrm{uniq}}$ the pairing is exchangeable, so the two orientations are equally likely,
and
\begin{equation*}
\Ex\!\left[\eijuniq\mid R_U,\{y_i,y_j\}\right]=\frac12\cdot\frac{2s}{s+s'}+\frac12\cdot\frac{2s'}{s+s'}=1,
\end{equation*}
with $0/0:=1$ keeping it exact where $s=s'=0$. By the tower property $\Ex[\eijuniq]=1$.
\end{proof}

\subsection{Deriving the invariance call}\label{app:einvderiv}
\textbf{The product law.} Conditional on the proposal data and the scope, the unit pairs
$(r_i,y_i)$ are drawn independently from one law, so for $i\ne j$ the pair $(r_i,y_j)$ follows the
product of its marginals, and $h(r,y)=r(y)$ decomposes under it as in \cref{eq:invmodel}, with
$a(r)=\Ex_y[h]-\mu$, $b(y)=\Ex_r[h]-\mu$ and $c=h-\mu-a-b$. Since $c$ has mean zero over $r$ for every
$y$ and over $y$ for every $r$, the cross terms in
$\tfrac12\Ex\big[(a(r_1)-a(r_2)+c(r_1,y)-c(r_2,y))^2\big]$ vanish, leaving
$\sigma_{\mathrm r}^2+\sigma_{\mathrm{ry}}^2$, which is \cref{eq:sr}. The diagonal is left out because
$(r_i,y_i)$ follows the joint law, not the product. For a unit-invariant hypothesis $a$ and $c$ vanish
and $S_{\mathrm r}=0$ (\cref{app:unitinv}).

\textbf{The estimate and its error.} For column $j$ let $v_j$ be the unbiased variance of
$\{M[i][j]:i\ne j\}$; then $\widehat S_{\mathrm r}=\tfrac1m\sum_j v_j$, with $m=|U|$, is the U-statistic
of $\tfrac12(M[i][j]-M[k][j])^2$ over distinct triples, unbiased for $S_{\mathrm r}$, and $\bar M$, the mean
of the $m(m-1)$ off-diagonal cells, is the U-statistic of $\tfrac12(M[i][j]+M[j][i])$ over distinct
pairs, unbiased for $\mu$. Deleting unit $u$ removes row $u$ and column $u$; each $v_j$ updates in
$O(1)$ from its column's sum and sum of squares, and $\bar M$ from the row and column sums, so the $m$
replicates $\widehat D^{(-u)}=\widehat S_{\mathrm r}^{(-u)}-(k\bar M^{(-u)})^2$ cost $O(m^2)$ together,
and $\widehat{\mathrm{se}}^2=\tfrac{m-1}{m}\sum_u\big(\widehat D^{(-u)}-\overline D\big)^2$ with
$\overline D$ their mean. $\widehat D$ is a smooth function of two U-statistics of independent units.
When the scores have finite fourth moments, so both kernels have finite second moments, and
$\widehat D$'s first-order projection has positive variance, it is asymptotically normal
\citep{hoeffding1948} and this jackknife estimates its variance consistently \citep{arvesen1969}; when $S_{\mathrm r}\ge(k\mu)^2$ the probability that
$\widehat D+z\,\widehat{\mathrm{se}}<0$ therefore tends to at most $\alpha$, and likewise for the variant
call.

\textbf{A finite-sample alternative.} For scores in $[0,1]$, as recall is, the kernel lies in
$[0,\tfrac12]$, so Hoeffding's inequality for
an order-3 U-statistic \citep{hoeffding1963} gives $\widehat S_{\mathrm r}\pm\tfrac12\sqrt{\ln(1/\alpha)/(2\lfloor m/3\rfloor)}$
with no approximation at all. At our ceiling that half-width alone is a median $2.4$ times $(k\bar M)^2$ over
the language-model tier's \emph{Shared} hypotheses, so at most one of the $21$ could be called shared even with no spread
at all; the call uses the jackknife, which calls all $21$ shared.

\textbf{How close the asymptotics are.} We simulated the call at its boundary,
$S_{\mathrm r}=(k\mu)^2$, with independent units and set-valued restrictions and outcomes---$|U|$ from
$50$ to $800$, $|y|\in\{3,6\}$, hit rates $0.1$ and $0.5$, $60{,}000$ draws. It called \emph{shared}
$1.2$--$4.7\%$ of the time per setting against $\alpha=5\%$, and at most $0.9\%$ at $1.1$ times the
ceiling; it called \emph{variant} $2.9$--$8.5\%$ of the time, $5.5\%$ pooled. The two sides differ
because the jackknife error moves with $\widehat D$: a high $\bar M$ raises the ceiling and the error
together, so at these sizes the call leans away from \emph{shared} and toward \emph{variant}.

\textbf{Conditions.} The bounds need units drawn independently given the
scope and the proposal data, restrictions that depend only on their unit's input and on proposal
randomness independent of the outcomes, and a hypothesis and scope chosen without reading a
certification outcome. The ceiling
is $k$ times the mean off-diagonal score $\mu$; $\bar M$ estimates it, and that error enters the
jackknife with the spread's. The one approximation is the normal one behind the jackknife bound,
roughest at small $|U|$ and small $\alpha$. None of this assumes a law for the scores or independence
between cells.

\section{Setting the tolerances}\label{app:tolerances}
The certifier takes four numbers from its user, and they are not alike: one is an error budget, two
define decision boundaries, and one is the noise ceiling the invariance call is tested under.
\Cref{tab:tolerances}
summarises them; the subsections say what each decides, which way its errors run, and how to pick
it.

\begin{table}[htbp]
\caption{The four user inputs. None carries a scale. $\alpha$, $\tau_{\mathrm{cov}}$ and $\gamma$
gate certification and err toward rejecting, so \textbf{none of the three has a failure direction
that produces a false certificate}. The ceiling is the exception: it fixes what counts as one
restriction, so too high a ceiling lets restrictions that genuinely differ certify as one
\emph{Shared} restriction, and it must be reported alongside any \emph{Shared} certificate.}
\label{tab:tolerances}
\centering\small
\begin{tabularx}{\linewidth}{@{}l>{\raggedright\arraybackslash}p{3.3cm}>{\raggedright\arraybackslash}p{3.5cm}Xc@{}}
\toprule
 & decides & err in this direction & how to set it & transfers \\
\midrule
$\alpha$ (as $1/\alpha$) & certify or not, on both e-values & too small: fewer certify; \emph{Matched} needs $|U|\ge14$ at $0.05$, $28$ at $0.001$ & your error budget for a single comparison & yes \\
\midrule
$\tau_{\mathrm{cov}}$ & whether enough units show the lift individually to certify & too high: real rules rejected---never miscertified & the share of units you require to carry the gain & yes \\
\midrule
$\gamma$ & whether a scope is the population: wider than $\gamma|T|$ reads as \emph{base rate} & too low: a real effect on a wide scope rejected---never miscertified & the share of $T$ beyond which a claim is about everyone: $0.9$ by default & yes \\
\midrule
$k$ & \emph{invariant} vs \emph{variant}: what counts as one restriction (\cref{sec:invariance}) & too high: an \emph{Unmatched} scope certified \emph{Shared}; too low: a shared restriction whose per-unit variation exceeds it is declined & the spread, as a fraction of the mean score, below which two restrictions count as one: $0.3$ by default & declared \\
\bottomrule
\end{tabularx}
\end{table}

\subsection{The level \texorpdfstring{$\alpha$}{alpha}}\label{app:alpha}
$\alpha$ is the per-comparison error budget of \cref{thm:iut} and enters as the single threshold
$1/\alpha$ against which both e-values are read; it is
per-comparison and not a batch quantity. $\alpha$ is also the level of the invariance call's
equivalence test and of the certified share, so one number sets every bar. One consequence is a \emph{precondition} rather than a
tolerance: every term is at most $2$---$\ejsig$ because scores lie in $[0,1]$, $\eijuniq$ by its form---so with $n$
terms $\mathrm{bet}\le\frac{1}{|\Lambda|}\sum_{\lambda\in\Lambda}(1+\lambda)^n$, which reaches $1/\alpha$ from
$n=7$ terms at $\alpha=0.05$, $10$ at $0.01$ and $14$ at $0.001$. So a \emph{Matched}
certificate needs at least $14$ units at $\alpha=0.05$, $20$ at $0.01$ and $28$ at $0.001$, its terms being pairs, and
a \emph{Shared} one $7$, $10$ and $14$, whatever the reference size. Neither floor involves
the scope's width (\cref{app:gamma}). At small scopes the
invariance call loses power: its jackknife error widens as $|U|$ shrinks, so more scopes are undecided
and declined rather than certified \emph{Shared}, and the normal approximation behind its bound is
roughest there.

\subsection{\texorpdfstring{$\tau_{\mathrm{cov}}$}{tau-cov}: the certified share}\label{app:choosecov}
A certificate claims the lift for most of the scope, and the e-values cannot carry that claim:
$\Esig$ rejects a null under which \emph{no} unit has lift, so a small subgroup can reject it alone,
and $\Euniq$ concerns the pairing, not how widely it holds. The certified share carries it. The comparators
$v_j$ of \cref{eq:cover} are assigned uniformly at random, one member of $V$ to a unit and without
reading an outcome. Unit $j$ \emph{has a gain} when it is likelier to win than to lose,
$\Prob(s_j=1)>\Prob(s_j=-1)$; a unit whose outcome is exchangeable with an out-of-scope one wins and loses equally
often. So $\Ex[s_j]\le1$ for every unit and $\Ex[s_j]\le0$ for a unit without a gain, and the expected mean
$\Ex[\bar s]$ is at most the share $c$ of units with a gain, whatever the law of the scores and however ties fall.

\textbf{Validity.} With units i.i.d., in scope and in the reference, the $s_j$ are independent across $j$. For a stake
$\lambda\in\Lambda$ every factor of \cref{eq:cover} is at least $1-\lambda>0$, and by independence and the AM--GM
inequality
\begin{equation*}
\Ex\prod_{j\in U}\Big(1+\lambda\frac{s_j-\tau}{1+\tau}\Big)=\prod_{j\in U}\Big(1+\lambda\frac{\Ex[s_j]-\tau}{1+\tau}\Big)
\le\Big(1+\lambda\frac{\Ex[\bar s]-\tau}{1+\tau}\Big)^{|U|}\le1\qquad\text{whenever }c\le\tau;
\end{equation*}
a mean over stakes is again an e-value, so the probability of reaching $1/\alpha$ is at most $\alpha$ by Markov's
inequality. Each factor decreases in $\tau$, so the $\tau$ at which \cref{eq:cover} holds form an interval $[0,\hat c]$,
and under $H_0^{\mathrm{cov}}:c\le\tau_{\mathrm{cov}}$, $\Prob(\hat c\ge\tau_{\mathrm{cov}})\le\alpha$. No asymptotics, no
model of the scores and no estimate of a chance rate enter.

\textbf{Why not the win rate.} The rate at which units outscore their comparators is not the share: a unit without a
gain still outscores its comparator by chance, so a minority plus chance wins can pass a test of the rate---with no
ties, a tenth of the units always winning and the rest winning half the time give a rate of $0.55$. Subtracting losses
removes chance by symmetry, since a unit without a gain loses as often as it wins. We checked the bound on the $136$
planted \emph{Matched} and \emph{Shared} samples, whose every unit has a gain, by turning a chosen fraction of $200$
units into exact non-carriers, each given an out-of-scope outcome, so that the true share is known: $\hat c$ exceeded it
in at most $0.22\%$ of the draws at any share, and sat $0.05$ to $0.11$ below it when carriers always win. A bound on
the win rate exceeded it in $49\%$ of the draws where no unit had a gain and, with ties broken by a coin, in every draw
up to a share of $0.7$.

\textbf{One comparator per unit.} A comparator per unit is what makes the signs independent. Comparing every unit with
the mean of one shared reference would tie every comparison to that reference: a unit without a gain can outscore a
mean more than half the time, and the count moves with the reference drawn---on the language-model tier one
hypothesis's share of units above the reference mean ranges from $0.37$ to $0.68$ across forty reference draws.

\textbf{What it does not show.} With one outcome per unit, a weak gain in every unit and a strong gain in a minority can
give the same signs, so the bound is tightest when carriers win whenever they can and conservative when gains are weak.
A failed test is not proof of a minority---a scope too small to show $\hat c\ge\tau_{\mathrm{cov}}$, or a weak gain
everywhere, fails it too---so \emph{Sparse} reads as coverage not shown. $\tau_{\mathrm{cov}}$ stays the user's
declaration of how many units must carry the gain, and being a fraction it transfers across scales. Like every gate it
can overrule an overwhelming mechanism statistic: \emph{simpson\_cat} is rejected at a median $\Euniq$ of $10^{22.8}$
(\cref{app:ablation}).

\textbf{On our tiers.} On the planted tiers coverage is shown for all $136$ samples of the two certifying types, at
certified shares of $0.60$ to $0.98$, and for none of the $68$ \emph{Sparse} ones, which reach $0.26$. Of the
$189$ labeled language-model hypotheses whose other tests pass, it is shown for $70$---every one the experts label
\emph{Matched} or \emph{Shared}, at certified shares of $0.50$ to $0.75$---and for none of the other $119$, which reach
$0.49$. That tier's restrictions are broad: a unit without a gain outscores its comparator almost a third of the time
(median chance rate $0.31$ for its \emph{Matched} hypotheses, against $0.04$ on the generator), yet the median
\emph{Matched} hypothesis wins $79\%$ and loses $11\%$ of its comparisons, beyond the $66\%$ and $16\%$ that half the
units always winning and the rest winning and losing by chance would give.

\subsection{\texorpdfstring{$\gamma$}{gamma}: where the population begins}\label{app:gamma}
Taking the reference from all of $T$ would put the scope inside its own baseline: a fraction $|U|/|T|$ of
every pool would be in scope, so the contrast $\Esig$ measures would shrink as the scope widens and vanish as it
approaches $T$. Drawing the reference from outside the scope keeps the contrast at any width, each term still
bounded by $2$ (\cref{app:alpha}).

What the reference leaves open is where a scope stops being one. A claim about nearly all of $T$ is
a claim about the population, which is what \emph{base rate} means, and what remains outside it is
too little of $T$ to stand for that population. $\gamma$ marks the point: a scope covering more than
$\gamma|T|$ is \emph{base rate} whatever its statistics say. It is a definition rather than a
measurement, set by that question alone and $0.9$ by default. Like the certified share it can only
reject, so it has no failure direction that produces a false certificate, and being a fraction of
$T$ it transfers across corpora unchanged.

\subsection{Sensitivity of the reported results}\label{app:tausens}
\textbf{The level.} $\alpha$ is the one constant \cref{thm:iut} depends on, so we sweep it; each cell counts decisions matching the target on the planted ($408$)
and language-model ($285$) tiers:

\begin{center}\small
{\setlength{\tabcolsep}{6pt}
\begin{tabular}{@{}rrc@{}}
\toprule
$\alpha$ & $1/\alpha$ & planted / language-model \\
\midrule
$0.20$ & $5$ & $405$ / $283$ \\
$0.10$ & $10$ & $406$ / $284$ \\
$0.05$ & $20$ & $407$ / $285$ \\
$0.01$ & $100$ & $408$ / $267$ \\
$0.001$ & $1000$ & $408$ / $255$ \\
\bottomrule
\end{tabular}}
\end{center}

Every gate reads $\alpha$: the invariance call through its confidence bounds and the certified share
through its level. Over the three loosest levels $N\times N$ changes at most twelve of the $408$ planted verdicts
and two of their decisions---two segment confounds whose uniqueness products sit just under $20$ certify
\emph{Matched} at $1/\alpha=5$---and fifty of the $285$ language-model verdicts, two of them decisions: at
$1/\alpha=5$ two hypotheses the experts label \emph{Sparse}, at certified shares of $0.49$ and $0.48$ at
$\alpha=0.05$, clear coverage and certify \emph{Matched}. Tighter levels are an
operating point: $N\times N$ matches all $408$ planted targets at both $0.01$ and $0.001$, and the language-model certificates
fall from $70$ to $52$ at $0.01$ and $40$ at $0.001$, every loss but one to coverage, whose bound widens at a
smaller level. Agreement with the target is not the quantity \cref{thm:iut} bounds, so
it cannot show whether the bound binds; the null draws of \cref{app:nullcheck} check that directly.

\textbf{The noise ceiling.} We swept $k$ on the generator, where the best value is $k=0.3$. Set there before the
language-model tier was built, it was used unchanged on the intervention and language-model tiers, where our certify
decisions match the target on all $108$ samples and all $285$ hypotheses---which suggests that the generator reflects
real data well in what $k$ depends on. Where a dataset's distribution can be measured or estimated and differs from ours,
$k$ can be reset the same way, without reading a certification outcome: simulate it as our generator does, plant shared
and variant restrictions, and take the $k$ that separates them best.

\textbf{Coverage.} $\tau_{\mathrm{cov}}$ is not tuned: a certificate claims that the gain
reaches most of the scope, and $\tau_{\mathrm{cov}}=0.5$ by that definition. A user who means a larger share
declares a larger $\tau_{\mathrm{cov}}$, and the certificate then claims that share instead.

\section{Null calibration}\label{app:nullcheck}
\textbf{Both e-values rarely reach $1/\alpha$ on nulls true by construction.} From the stored score matrices we draw,
for each hypothesis, $20$ nulls of each kind: every unit's own outcome relabeled uniformly within its pool, which makes
$H_0^{\mathrm{sig}}$ true with equality, and every pair's outcomes swapped with probability one half, which makes
$H_0^{\mathrm{uniq}}$ true. Over the three tiers, $\Esig$ reaches $1/\alpha$ $31$ times in $13{,}360$ draws and
$\Euniq$ $25$ times in $13{,}860$, at most $0.29\%$ on any tier, against $\alpha=5\%$.

\textbf{No lift without exchangeability.} A relabeling keeps in-scope and out-of-scope scores exchangeable, the one
case a ratio term also covers. We therefore also draw, $20$ times per hypothesis, each unit's own score from the
two-point law on adjacent score values whose mean is that unit's mean score over its reference, and its block from that
reference: no lift, but in-scope scores less spread than out-of-scope ones. $\Esig$ reaches $1/\alpha$ $8$ times in the
$13{,}360$ draws, at most $0.14\%$ on any tier. A ratio term, a unit's score over its pool's
mean, reaches it on $14\%$ of the generator draws, $0.3\%$ of the intervention draws and $97\%$ of the language-model
draws, most often where out-of-scope scores are dense. Neither kind of draw checks that the units were drawn
independently, the assumption the guarantees rest on (\cref{app:theory}).

\section{The language-model tier}\label{app:indep}
The planted tiers' hypotheses are ours, so for this tier language-model sessions wrote them and human experts labeled
them.

\textbf{Corpus.} The Electronics category of Amazon Reviews 2023 \citep{hou2024amazon}. A user's purchases are their
distinct items, each at its latest timestamp, reviewed in 2022--2023 UTC and taken in time order. Every user with at
least six is kept: one with $n\ge10$ has the last five as outcome and the rest as history, one with $6\le n\le9$ the
last $\lfloor n/2\rfloor$, so $|y|$ is $3$ to $5$ and a unit's history and outcome are disjoint in item and time. A hash
of the user id splits them into a proposal half of $39{,}667$ and a certification half of $|T|=39{,}510$. Item records
carry title, store, category path, price, brand, manufacturer, model, compatibility, feature bullets, average rating and
rating count. A restriction is a test on one item record with a value in $[0,1]$, and it scores an outcome by the mean
value over the outcome's items, so an item list is the special case of a membership test.

\textbf{Proposers.} Each session received a task statement, a data interface, the proposal half with its outcomes and
the certification half's histories, to size scopes; it saw no certification outcome and nothing about our method or
classes, and was told to read nothing else on the machine. The task asked for hypotheses correct on the proposal half,
meaning a significant gain over out-of-scope users that holds for more than half of the users in scope, every scope
between $400$ users and $90\%$ of the certification half, and no two near-duplicates. Some sessions wrote new
hypotheses; others took hypotheses that sessions knowing neither our method nor our classes had written
with a list of items as the restriction, and restated them as tests on item properties with every scope kept exactly;
some also ran under a looser wording that
asked only for a gain clearly above a sensible baseline and offered a fit readout on the proposal half. Every set was
frozen before any method ran on it.

\textbf{Evaluation.} Every method evaluates each hypothesis on the certification half, whose outcomes no proposer saw,
at $\alpha=0.05$, $\tau_{\mathrm{cov}}=0.5$ and $k=0.3$; each scope's $|U|=400$ units and $|V|=2000$
references are drawn uniformly, so $\ell=5$. The composite's unique-gain CRT uses a resampling law written for each
hypothesis, $285$ in all. Language-model sessions wrote each law from the hypothesis's code, with the proposal half and
the certification half's histories but no certification outcome. A law holds fixed what puts a unit in the scope and
redraws what the restriction is built from, sampling it from a model fitted on the histories of the certification
half's users other than the scored ones; each law was checked to keep its redraws in scope and well formed, and to
repeat them exactly from the same seed. Held-out validation compares each in-scope unit's hit rate with that of a random item set of the same size; for a
test the comparison is the same test on as many items drawn at random from the catalogue, which it passes at its
catalogue pass rate, the chance a random set of that size would give.

\textbf{Labels.} Human experts labeled $285$ of the hypotheses with their verdicts, by the definitions of
\cref{tab:classes}; \cref{tab:indep} gives the count of each. They are given each hypothesis's description and its
$\mathrm{in\_scope}$ and $\mathrm{apply}$, the proposer's own code, and see no method's output until every label is
set. Whether a hypothesis is shared they read off its restriction, by whether it
depends on anything about the unit: \emph{the item is under Camera \& Photo} reads nothing about the unit and is one
restriction, while \emph{the item comes from a store the user bought from before} differs from unit to unit. Whether it
has a gain, and whether the gain reaches most of the scope, they check by running it on a separate
dataset built for them, independent of the certification half the methods score, and inspecting users by hand.

\textbf{Where our verdicts miss.} Every decision matches the labels; seven verdict names do not. Four hypotheses the
experts label \emph{Unmatched} pass the call as shared: one has a gain and fails coverage, and is named \emph{Sparse};
three do not clear $1/\alpha$ on $\Esig$, at $1.7$, $3.8$ and $10.4$, and are named \emph{base rate}. One labeled
\emph{Unsupported} is named \emph{base rate}: both mean no gain, and the call reads its restrictions as one within the
ceiling. Two labeled \emph{Matched} are named \emph{Shared}, both certified: their uniqueness products, $9.8$ and $13.3$,
stop short of $20$, so the call routes them as one restriction.

\section{Batch control over a stream of hypotheses}\label{app:ebh}
\cref{thm:iut} bounds the error of \emph{one} certification, and a proposer emitting many
accumulates them in proportion.

$\Esig$ and $\Euniq$ are e-values, which is what makes the batch question cheap to answer:
e-BH \citep{wang2022false} takes e-values directly and controls the false discovery rate under
\emph{arbitrary} dependence---the regime here, since the $285$ hypotheses
are read off one shared outcome multiset. It selects the top $k^{\ast}$,
$k^{\ast}=\max\{k: e_{(k)}\ge m/(k\alpha)\}$. That bar is far above
$1/\alpha$ except in the limit: at $m=285$ and $\alpha=0.05$ it asks $5{,}700$ of the largest e-value,
$25.3$ of the $225$th, and reaches $20$ only at $k=285$. Every certificate claims a lift, which $\Esig$ carries. A
\emph{Matched} certificate also claims the lift is matched to the unit; its two-part claim is read on
$\min(\Esig,\Euniq)$, which is an e-value under either part of its null, since it is at most whichever e-value that part
makes valid.

\begin{center}\small
{\setlength{\tabcolsep}{3pt}
\begin{tabular}{@{}lccc@{}}
\toprule
 & clears $1/\alpha$ & survives e-BH & largest e-value \\
\midrule
$\Esig$ & $228$ & $\mathbf{225}$ & $10^{66.4}$ \\
\quad by verdict: \emph{Matched} / \emph{Shared} / \emph{Unmatched} / \emph{Sparse} & $47/23/39/119$ & $\mathbf{47}/\mathbf{23}/37/118$ & \\
$\min(\Esig,\Euniq)$ & $113$ & $\mathbf{112}$ & $10^{38.2}$ \\
\quad by verdict: \emph{Matched} / \emph{Sparse} & $47/66$ & $\mathbf{46}/66$ & \\
\bottomrule
\end{tabular}}
\end{center}

\textbf{The lift of every certificate survives the batch bar, and the two-part claim of all but one \emph{Matched}
certificate.} All $47$ \emph{Matched} and $23$ \emph{Shared}
certificates clear the e-BH threshold on $\Esig$, as do $37$ of the $39$ \emph{Unmatched} and $118$ of the $119$
\emph{Sparse} rejects; none of the $57$ \emph{base rate}, \emph{Unsupported} and \emph{Subthreshold-Matched}
readings clears $1/\alpha$. On $\min(\Esig,\Euniq)$, $46$ of the $47$ \emph{Matched} certificates clear the
e-BH threshold; the one that does not has $\Euniq=45.3$, against a bar of $50.9$ at $k^{\ast}=112$.

\end{document}